\documentclass[11pt,letterpaper]{article}

\usepackage{amssymb,amsmath,amsfonts}
\usepackage{graphicx,xcolor,enumitem}
\usepackage{epsfig}
\usepackage{amsthm}
\usepackage{bm}
\usepackage{multirow}
\usepackage{subcaption}
\usepackage[numbers]{natbib}
\usepackage{csquotes}
\renewcommand{\mkbegdispquote}[2]{\itshape}
\usepackage[twoside, hmarginratio=1:1, vmarginratio=1:1, left=1in,top=1in]{geometry}

\RequirePackage[breaklinks=true, hidelinks]{hyperref}
\usepackage{breakcites}

\usepackage{algorithm, algorithmicx}
\usepackage{algpseudocode}

\usepackage{accents}

\newcommand{\cB}{\mathcal{B}}
\newcommand{\E}{\mathbb{E}}
\newcommand{\R}{\mathbb{R}}
\newcommand{\p}{\mathbb{P}}
\newcommand{\q}{\mathbb{Q}}
\newcommand{\cL}{{\mathcal L}}
\newcommand{\cS}{{\mathcal S}}
\newcommand{\cF}{{\mathcal F}}
\newcommand{\cP}{{\mathcal P}}
\newcommand{\cN}{{\mathcal N}}
\newcommand{\cG}{{\mathcal G}}
\newcommand{\cZ}{{\mathcal Z}}

\newcommand{\tr}{{\rm{tr}}}

\newcommand{\cT}{{\mathcal T }}

\newcommand{\cV}{{\mathcal V}}
\newcommand{\cW}{{\mathcal W}}

\newcommand{\cX}{{\mathcal X}}
\newcommand{\cY}{{\mathcal Y}}

\newcommand{\cC}{{\mathcal C}}
\newcommand{\cH}{{\mathcal H}}
\newcommand{\cR}{{\mathcal R}}
\newcommand{\cM}{{\mathcal M}}

\newtheorem{theorem}{Theorem}

\newtheorem{assumption}[theorem]{Assumption}

\newtheorem{corollary}[theorem]{Corollary}

\newtheorem{definition}[theorem]{Definition}

\newtheorem{lemma}[theorem]{Lemma}

\newtheorem{proposition}[theorem]{Proposition}

\theoremstyle{definition}

\numberwithin{equation}{section}
\numberwithin{theorem}{section}

\begin{document}

\title{Fitted value iteration methods for bicausal optimal transport}

\author{
	Erhan Bayraktar \thanks{Department of Mathematics, University of Michigan, Ann Arbor, Email: erhan@umich.edu. Erhan Bayraktar is partially supported by the National Science Foundation under grant DMS-2106556 and by the Susan M. Smith chair.}
	\and Bingyan Han \thanks{Thrust of Financial Technology, The Hong Kong University of Science and Technology (Guangzhou). Email: bingyanhan@hkust-gz.edu.cn. Bingyan Han is partially supported by The Hong Kong University of Science and Technology (Guangzhou) Start-up Fund G0101000197, the Guangzhou-HKUST(GZ) Joint Funding Program (No. 2024A03J0630), and the National Natural Science Foundation of China (Grant No. 12401621).  This work began when Bingyan Han was a postdoctoral researcher in the Department of Mathematics at the University of Michigan. He expresses gratitude to the University of Michigan for providing support and an atmosphere conducive to this work.}
}

\date{July 26, 2025}
\maketitle

\begin{abstract}
    We develop a fitted value iteration (FVI) method to compute bicausal optimal transport (OT) where couplings have an adapted structure. Based on the dynamic programming formulation, FVI adopts a function class to approximate the value functions in bicausal OT. Under the concentrability condition and approximate completeness assumption, we prove the sample complexity using (local) Rademacher complexity. Furthermore, we demonstrate that multilayer neural networks with appropriate structures satisfy the crucial assumptions required in sample complexity proofs. Numerical experiments reveal that FVI outperforms linear programming and adapted Sinkhorn methods in scalability as the time horizon increases, while still maintaining acceptable accuracy. 
	\\[2ex] 
	\\[2ex]
\end{abstract}

\section{Introduction}
Optimal transport (OT) has emerged as a valuable tool for comparing probability distributions, finding applications in diverse disciplines including control and optimization \cite{acciaio2021cournot,bayraktar2021transport,eckstein2021computation,backhoff2017causal,cuturi2013sinkhorn}. OT possesses an intuitive geometric interpretation, quantifying the costs involved in transforming one probability distribution into another by moving point masses. Although the concept of OT originated in the works of Monge in 1781 and Kantorovich in the 1940s \cite[Chapter 3]{villani2009optimal}, its widespread adoption was catalyzed by the development of an efficient computational method utilizing entropic regularization, known as the Sinkhorn algorithm \cite{cuturi2013sinkhorn}, together with variants in \cite{altschuler2017near,lin2022efficiency}. These advancements have alleviated the computational burden associated with OT, contributing to its increased popularity.

Temporal data play a ubiquitous role in various domains, such as control \cite{bayraktar2018martingale,acciaio2021cournot}, optimization \cite{backhoff2017causal}, and computer science \cite{xu2020cot}. However, the conventional OT framework fails to account for the temporal structure and information flow inherent in such data. To address this limitation, a recent surge of research introduced causal OT with a non-anticipative framework \cite{lassalle2013causal,backhoff2017causal}. In essence, causal OT requires that when the past of one process, $X$, is given, the past of another process, $Y$, should be independent of the future of $X$ under the transport plan. If the same condition holds after interchanging the role of $X$ and $Y$, the transport plan is called bicausal; see \cite{pflug2012distance,xu2020cot,acciaio2021cournot} for recent applications.

The classical OT method is hindered by its high computational costs, and the inclusion of the temporal dimension in causal OT further exacerbates the computational challenges. Addressing this issue, \cite{eckstein2022comp} recently proposed an adapted Sinkhorn algorithm tailored specifically for causal OT, based on the Bregman iterative formulation of the Sinkhorn algorithm \cite[Remark 4.8]{peyre2019computational}. As with previous approaches, continuous distributions are approximated using either empirical samples or quantization techniques. However, it is important to note that even the classical Sinkhorn algorithm requires $O(N^2)$ operations for a sample size of $N$ \cite{genevay2019sample}, which may not scale well under large sample sizes, particularly in high-dimensional scenarios. Furthermore, the adapted Sinkhorn algorithm in \cite[Lemma 6.4]{eckstein2022comp} enumerates all possible scenarios at each time step, and this enumeration grows rapidly as the time horizon increases.

The computational burden arising from discretization is not unique to OT, but rather a common challenge also encountered in dynamic programming (DP) applications, such as optimal control, vehicle routing, and healthcare \cite{powell2007approximate}. Discretization often leads to an immense number of states, rendering the Bellman equation computationally intractable. To address this issue, approximate DP and reinforcement learning (RL) techniques have been developed to provide approximations of value functions or optimal policies. The use of neural networks as function approximators for action-value functions enables handling tasks with high-dimensional inputs \cite{mnih2015human}. It is worth noting that discretization approaches often neglect the structural properties inherent in value functions and optimal policies. For example, under certain conditions, the Wasserstein distance $W_p(\mu(dx_{1:T}|x'_0), \nu(dy_{1:T}|y'_0))$ between two conditional distributions $\mu(dx_{1:T}|x'_0)$ and $\nu(dy_{1:T}|y'_0)$ should not vary too much compared with $W_p(\mu(dx_{1:T}|x_0), \nu(dy_{1:T}|y_0))$, where $(x'_0, y'_0)$ is close to $(x_0, y_0)$. Knowledge about $W_p(\mu(dx_{1:T}|x_0), \nu(dy_{1:T}|y_0))$ can be beneficial in determining $W_p(\mu(dx_{1:T}|x'_0), $ $ \nu(dy_{1:T}|y'_0))$.

Notably, bicausal OT can be solved by DP \cite[Corollary 3.3]{backhoff2017causal}, which opens up possibilities for leveraging learning-based methods in optimal control. In this regard, transport plans serve as control policies. We employ the fitted value iteration (FVI) technique from offline RL \cite{gordon1999approx,munos2008finite,duan2021risk}, which also eliminates the need for knowing the underlying distributions. In contrast to prior research \cite{genevay2016stochastic,seguy2018large}, we utilize neural networks as approximators for the value functions in the primal DP formulation. Also different with \cite{pichler2022nested}, our method is not limited to the Sinkhorn algorithm. Consequently, we derive the following theoretical and numerical results:

To analyze the sample complexity and error bounds of FVI, we adopt the concentrability condition introduced in \cite{munos2003error,munos2008finite,duan2021risk,fan2020theoretical,chen2019information} to establish a connection between value suboptimality and Bellman errors in Lemma \ref{lem:subopt}. Leveraging standard tools in (local) Rademacher complexity \cite{mohri2018foundations,bartlett2005local} and the convergence of empirical measures in Wasserstein distance \cite{fournier2015rate,genevay2019sample}, we derive the sample complexity for FVI in Theorems \ref{thm:light_RC} and \ref{thm:LRC_light} and Corollaries \ref{cor:entropy_RC} and \ref{cor:LRC_entropy}. Additionally, to demonstrate the feasibility of the crucial approximate completeness assumption \ref{assum:complete} and the specific Lipschitz assumption \ref{assum:Lip}, we establish that neural networks with rectified linear unit (ReLU) and sigmoid activation functions satisfy these assumptions when appropriately structured. Notably, recent advancements in approximating H\"older functions with neural networks in \cite{schmidt2020,langer2021approximating} provide explicit bounds on the approximate completeness assumption \ref{assum:complete}. Moreover, the fixed point of a sub-root function in Assumption \ref{assum:sub-root} can be bounded on the order of $\ln(N)/N$ using the covering number technique.

Numerically, we conducted a comparison between FVI and linear programming (LP) and adapted Sinkhorn algorithms. With Gaussian data, we observed that the LP and adapted Sinkhorn methods provided highly accurate estimates with low variances. Additionally, these methods demonstrated faster convergence than FVI for short time horizons. However, as the time horizon increased, FVI exhibited superior scalability compared to the LP and adapted Sinkhorn methods. Notably, for time horizons $T \geq 20$, the LP and adapted Sinkhorn methods failed to converge within a reasonable time frame, while FVI obtained estimations with satisfactory accuracy. Moreover, our implementation offered a unified approach to handling multidimensional data.

The structure of this paper is as follows: Section \ref{sec:notation} and Section \ref{sec:bic} introduce the notation and bicausal OT, respectively. Section \ref{sec:FVI} presents the FVI algorithm for bicausal OT, while Section \ref{sec:sam_complex} provides the proof of sample complexity. In Section \ref{sec:NN}, we demonstrate that the ReLU and sigmoid neural networks satisfy crucial assumptions outlined in Section \ref{sec:sam_complex}. We present the numerical comparison in Section \ref{sec:numerics}. Our code is available at \url{https://github.com/hanbingyan/FVIOT}. The Appendix gives technical proofs with auxiliary results used.

\subsection{Notation and preliminaries}\label{sec:notation}

We summarize frequently used notations here. Denote $\lceil x \rceil$ as the least integer greater than or equal to $x$. For two numbers $a$ and $b$, let $a \vee b = \max\{a, b\}$. For a function $f: \cS \rightarrow \R$, $\| f \|_\infty := \sup_{s \in \cS} |f(s)|$ is the uniform norm. Denote $\| W \|_\infty$ as the maximum-entry norm of a vector or matrix $W$.  Denote $\| W \|_0$ as the number of non-zero entries of $W$. For two real-valued functions $f(\cdot)$ and $g(\cdot)$, we write $f(\cdot) \lesssim g(\cdot)$ if $f(\cdot) \leq C \cdot g(\cdot)$ with an absolute constant $C$ independent of function parameters. Denote $\cC^\infty$ as the set of infinitely differentiable functions.

Let a positive integer $T$ be the finite number of periods. For each $t \in \{0, 1, ..., T-1, T\}$, suppose $\cX_t$ is a closed subset of $\R^d$. $\cX_t$ stands for the range of the process at time $t$. $\cX : = \cX_{0:T} = \cX_0 \times \ldots \times \cX_T$ is a closed subset of $\R^{(T+1) \times d}$. Denote the set of all Borel probability measures on $\cX$ as $\cP(\cX)$. We write $\mu(dx_{t+1:T}|x_{0:t})$ as the regular conditional probability kernel of $x_{t+1:T}$ given $x_{0:t}$, which is uniquely determined in a suitable way \cite[Theorem 10.4.14 and Corollary 10.4.17]{bogachev2007measure}. 

With a fixed $p \in [1, \infty)$, we introduce the metric as $d_\cX(x_{0:T}, x'_{0:T}) = \left[ \sum^T_{t=0} |x_t - x'_t|^p \right]^{1/p}$ for $x, x' \in \cX$ and equip $\cX$ with the corresponding Polish topology. The Wasserstein space of order $p$ is given by
\begin{equation*}
	\cP_p(\cX) := \left\{ \mu \in \cP(\cX) \Big| \int_{\cX} d_\cX(x_{0:T}, \bar{x}_{0:T})^p \mu(dx) < \infty \right\}
\end{equation*}
for some fixed $\bar{x}_{0:T} \in \cX$. 


The Wasserstein distance has been widely used as a metric on probability spaces. Consider two Borel probability measures $\mu, \mu' \in \cP(\cX)$. A coupling $\gamma \in \cP(\cX \times \cX)$ of $\mu$ and $\mu'$ is a Borel probability measure that admits $\mu$ and $\mu'$ as its marginals on $\cX$. Denote $\Pi(\mu, \mu')$ as the set of all the couplings. The Wasserstein distance of order $p$ between $\mu$ and $\mu'$ is given by
\begin{equation}\label{eq:Wp}
	W_p(\mu, \mu') = \left( \inf_{\gamma \in \Pi(\mu, \mu')} \int_{\cX \times \cX} d_\cX(x, x')^p \gamma (dx, dx') \right)^{1/p}.
\end{equation}

We introduce another closed set $\cY = \cY_{0:T} = \cY_0 \times ... \times \cY_T$ similarly. Denote a metric on $\cY$ as $d_\cY(y_{0:T}, y'_{0:T}) = \left[ \sum^T_{t=0} |y_t - y'_t|^p \right]^{1/p}$ for $y, y' \in \cY$ and equip $\cY$ with the corresponding Polish topology. Under the metric $d_\cY$, we introduce $\cP_p(\cY)$ as the Wasserstein space of order $p$ and $W_p(\nu, \nu')$ as the Wasserstein metric on $\cP_p(\cY)$. 

For the product space $\cX \times \cY$, we define the corresponding Wasserstein space of order $p$ as
\begin{equation*}
	\cP_p(\cX \times \cY) := \left\{ \pi \in \cP(\cX \times \cY) \Big| \int_{\cX \times \cY} d((x, y), (\bar{x}, \bar{y}))^p \pi(dx, dy) < \infty \right\},
\end{equation*}
with the metric $d$ given by
\begin{align*}
	d((x, y), (\bar{x}, \bar{y})) = [d_{\cX}(x, \bar{x})^p + d_{\cY}(y, \bar{y})^p]^{1/p}.
\end{align*}
For two probability measures $\pi, \pi' \in \cP_p(\cX \times \cY)$, the Wasserstein distance between them is denoted by
\begin{equation*}
	W_p(\pi, \pi') = \left( \inf_{\gamma \in \Pi(\pi, \pi')} \int_{\cX \times \cY \times \cX \times \cY} d((x, y), (x', y'))^p \gamma (dx, dy, dx', dy') \right)^{1/p}.
\end{equation*}

\subsection{Bicausal optimal transport}\label{sec:bic}

Consider two probability measures $\mu \in \cP(\cX)$ and $\nu \in \cP(\cY)$. Recall $\Pi(\mu, \nu)$ is the set of all the couplings that admit $\mu$ and $\nu$ as marginals. Suppose transporting one unit of mass from $x$ to $y$ incurs a cost of $c(x,y)$. A generic OT problem is formulated as 
\begin{equation*}
	\cW(\mu, \nu) := \inf_{ \pi \in \Pi(\mu, \nu)} \int_{\cX \times \cY} c(x, y) \pi(dx, dy),
\end{equation*}
with the Wasserstein distance of order $p$ in \eqref{eq:Wp} as a special case. 

If the data have a temporal structure as $x = (x_0, ..., x_t, ..., x_T)$ and $y = (y_0, ..., y_t, ..., y_T)$, not all couplings $\pi \in \Pi(\mu, \nu)$ will make sense. A natural requirement of the transport plan $\pi(x, y)$ should be the non-anticipative condition. Informally speaking, if the past of $x$ is given, then the past of $y$ should be independent of the future of $x$ under the measure $\pi$. Mathematically, it means a transport plan $\pi$ should satisfy
\begin{equation}\label{eq:causal}
	\pi (dy_t | x_{0:T}) = \pi (dy_t | x_{0:t}), \quad  t = 0, ..., T-1, \quad \text{$\pi$-a.s.}
\end{equation}
The property \eqref{eq:causal} is known as the causality condition and the transport plan satisfying \eqref{eq:causal} is called {\it causal} by \cite{lassalle2013causal}. If the same condition holds when we exchange the positions of $x$ and $y$, then the transport plan is called bicausal. Denote $\Pi_{bc}(\mu, \nu)$ as the set of all bicausal transport plans between $\mu$ and $\nu$. The bicausal OT problem considers the optimization over $\Pi_{bc}(\mu, \nu)$ only:
\begin{equation}\label{eq:bc}
	\cW_{bc}(\mu, \nu) := \inf_{ \pi \in \Pi_{bc}(\mu, \nu)} \int_{\cX \times \cY} c(x, y) \pi(dx, dy).
\end{equation}
For applications of causal and bicausal OT, see \cite{xu2020cot,acciaio2021cournot,pflug2012distance} for an incomplete list.

In this paper, we consider two probability measures $\mu$ and $\nu$ with symmetric positions and focus on the bicausal transport plans only. For a given transport plan $\pi \in \Pi(\mu, \nu)$, we can decompose $\pi$ in terms of successive regular kernels:
\begin{align}
	\pi(dx_{0:T}, dy_{0:T}) = & \bar{\pi} (dx_0, dy_0) \pi(dx_1, dy_1 | x_0, y_0) \pi(dx_2, dy_2 | x_{0:1}, y_{0:1}) ... \nonumber \\
	& \pi(dx_T, dy_T | x_{0:T-1}, y_{0:T-1}), \label{eq:decomp}
\end{align}
which is uniquely determined in a suitable way \cite[Theorem 10.4.14 and Corollary 10.4.17]{bogachev2007measure}. By \citet[Proposition 5.1]{backhoff2017causal}, $\pi$ is a bicausal transport plan if and only if
\begin{itemize}
	\item[(1)] $\bar{\pi} \in \Pi(p^1_* \mu, p^1_* \nu)$, and
	\item[(2)] for each $t = 0, ... , T-1$ and $\pi$-almost every path $(x_{0:t}, y_{0:t})$, the following condition holds:
	\begin{equation*}
		\pi(dx_{t+1}, dy_{t+1}| x_{0:t}, y_{0:t}) \in \Pi( \mu(dx_{t+1}| x_{0:t}), \nu(dy_{t+1}|y_{0:t})). 
	\end{equation*} 
\end{itemize}
$p^1_* \mu$ (resp. $p^1_* \nu$) is the pushforward of $\mu$ (resp. $\nu$) by the projection $p^1$ onto the first coordinate.

To ease the notation, we denote $s_t := (x_t, y_t)$ which is interpreted as the state in RL. We adopt the convention that $s_{0:t} = (x_{0:t}, y_{0:t})$. Denote $\cS := \cX \times \cY$ as the range of states. For the completion of notation, we interpret $\pi(ds_{1:0} | s_0) = \delta_{s_0}$ as the Dirac measure at $s_0$.  For simplicity, denote the set of couplings as $\Pi(\mu^t, \nu^t, s_{0:t}) := \Pi(\mu(dx_{t+1}| x_{0:t}), \nu(dy_{t+1}|y_{0:t}))$. With a given initial state $s_0$ and a cost function depending on $s_{1:T}$, we introduce the initial value function for bicausal OT \eqref{eq:bc} as
\begin{align}\label{eq:V0}
	V_0(s_0) := \inf_{\pi \in \Pi_{bc}(\mu, \nu, s_{0})}  \int c(s_{1:T}) \pi(ds_{1:T} | s_0).
\end{align}

\section{Fitted value iteration}\label{sec:FVI}

\subsection{Motivation}
In line with the conventional OT framework, computational methods have been developed for bicausal OT, employing both primal and dual formulations. The inclusion of causality introduces a linear constraint on transport plans. The strong duality results for causal and bicausal OT were established by \citet[Theorem 2.6]{backhoff2017causal} and \citet[Corollary 3.3]{backhoff2017causal}, respectively. For the discrete case of causal OT, \citet[Algorithm 1]{acciaio2021cournot} incorporated basis functions for causality testing and then employed the Sinkhorn algorithm, while for the continuous case, \cite{xu2020cot} utilized neural networks to model the test functions instead. It is worth noting that the Sinkhorn algorithm can be interpreted as Bregman iterative projections \cite[Remark 4.8]{peyre2019computational}. With this observation, \cite{eckstein2022comp} introduced an adapted version of the Sinkhorn algorithm. Besides, \cite{pichler2022nested} also proposed a nested Sinkhorn algorithm. For the distinctions between these two methods, please refer to \citet[Remark 6.12]{eckstein2022comp}. Finally, analogous to the classical OT, the LP formulation is also available for discrete causal and bicausal OT \cite[Lemma 3.11]{eckstein2022comp}.

All these methods mainly handle discrete distributions or samples from continuous distributions.  The LP and Sinkhorn algorithms require $O(N^3\log N)$ and $O(N^2)$ operations respectively, based on a sample of size $N$ \citep{genevay2019sample}. However, even the Sinkhorn algorithm may encounter scalability issues when dealing with a large number of samples. Furthermore, the computational burden in causal and bicausal OT is exacerbated by the temporal dimension $T$. The adapted Sinkhorn algorithm discussed in \citet[Lemmas 6.2 and 6.6]{eckstein2022comp} necessitates iterating over all possible scenarios at each time $t$, with the number of scenarios rapidly growing as the time horizon increases. It's important to note that these challenges are not unique to causal and bicausal OT but also arise in other fields, including DP. A crucial observation is that for a good cost function $c$ and marginals $\mu$ and $\nu$, if two states $s_0$ and $s'_0$ are close, then the values $V_0(s_0)$ and $V_0(s'_0)$ defined in \eqref{eq:V0} should also be close. These structural properties are ignored when using matrices to model value functions.

Given the DP formulation for bicausal OT \eqref{eq:DPP}, we posit that techniques capable of addressing the computational challenges in DP may also find applicability in bicausal OT. As the state space can be vast, it often becomes necessary to employ approximation techniques. Approximate DP and RL frequently leverage specific functions to approximate the value functions or optimal policies. In this study, we employ FVI to compute the bicausal OT costs, drawing inspiration from its successful applications in \cite{munos2008finite} and batch RL \citep{duan2021risk}. FVI provides a unified approach for handling continuous distributions and high-dimensional data, exhibiting superior scalability compared to previous methods. Our FVI algorithm operates in the primal form, utilizing neural networks to approximate the value functions. In contrast, prior literature has utilized neural networks to model potential functions in the dual formulation \citep{genevay2016stochastic,seguy2018large}.  

\subsection{Algorithm}

\cite{backhoff2017causal} proved the DP principle for bicausal transport. Consider the recursive formulation
\begin{equation}\label{eq:DPP}
	\left \{
	\begin{aligned}
		V_t(s_{0:t}) & = && \inf_{\pi(ds_{t+1}| s_{0:t}) \in \Pi(\mu^t, \nu^t, s_{0:t}) } \int V_{t+1} ( s_{0:t}, s_{t+1}) \pi(ds_{t+1}| s_{0:t}), \quad t =0, ..., T-1,\\
		V_T(s_{0:T}) & = &&  c(s_{1:T}).
	\end{aligned} \right.
\end{equation} 
By \citet[Proposition 5.2]{backhoff2017causal}, if the cost function $c(s_{1:T})$ is lower semicontinuous and bounded from below, $\mu$ and $\nu$ are  successively weakly continuous in the sense that for each $t = 0, ..., T-1$,
\begin{align*}
	x_{0:t} \mapsto \mu(dx_{t+1:T}|x_{0:t}) \text{  and  }  y_{0:t} \mapsto \mu(dy_{t+1:T}|y_{0:t})
\end{align*}
are continuous with respect to the weak topologies on $\cP(\cX)$ and $\cP(\cY)$. Then there is a bicausal optimizer and the recursive formulation \eqref{eq:DPP} is well-defined. The value function $V_t$ is lower semicontinuous. We will consider a continuous and bounded cost function in Assumption \ref{assum:bound_K}.

\begin{algorithm}[h]
	\small
	\begin{algorithmic}[1]
		\State {\bf Input:} Sampling distribution $\beta$; a fixed initial state $s_0$; initial neural network $f_t(s_{0:t})$, $t = 0, 1, ..., T-1$. Denote $f_T(s_{0:T}) = \hat{f}_T (s_{0:T}) = c(s_{1:T})$. \\
		\For{$t = T - 1, T - 2, ..., 0$}
		\State Sample a mini-batch of $N$ paths $s^{n_t}_{0:t} = (x^{n_t}_{0:t}, y^{n_t}_{0:t})$, $n_t = 1, ..., N$, according to $\beta(ds_{1:t}| s_0)$. Return $s_0$ when $t=0$.
		\For{$n_t = 1, ..., N$} 
		\State Construct empirical measures $\hat{\mu}_B(dx_{t+1}|x^{n_t}_{0:t})$ and $\hat{\nu}_B(dy_{t+1}|y^{n_t}_{0:t})$ with size $B$.
		\State Calculate the empirical OT to approximate the value function at $s^{n_t}_{0:t}$:
		\begin{align}\label{eq:empi_OT}
			& \hat{V}_t(s^{n_t}_{0:t}) \leftarrow \inf_{\pi(ds_{t+1} | s^{n_t}_{0:t}) \in \Pi(\hat{\mu}^t_B, \hat{\nu}^t_B, s^{n_t}_{0:t})} \int \hat{f}_{t+1} (s^{n_t}_{0:t}, s_{t+1}) \pi(ds_{t+1} | s^{n_t}_{0:t}).
		\end{align}
		\EndFor
		\State Minimize the empirical loss by performing $G$ gradient descent steps on parameters of $f_t$. Obtain
		\begin{equation}\label{eq:emp_loss}
			\hat{f}_t \leftarrow \arg\min_{f_t \in \cF_t} \frac{1}{N} \sum^N_{n_t=1} \left( f_t(s^{n_t}_{0:t}) - \hat{V}_t (s^{n_t}_{0:t}) \right)^2.
		\end{equation}
		\EndFor
		\State {\bf Output:} Value function approximators by neural network $\hat{f}_t(s_{0:t})$, $t = 0, 1, ..., T-1$.
	\end{algorithmic}
	\caption{Fitted value iteration}
	\label{algo:fvi}
\end{algorithm}

Algorithm \ref{algo:fvi} for bicausal OT is inspired by FVI in RL \citep{munos2008finite,duan2021risk}. To learn the true value function $V_t$, FVI adopts an approximator $f_t (s_{0:t}) \in \cF_t$, where $\cF_t$ is certain function class. In particular, $f_T(s_{0:T}) := c(s_{1:T})$. Denote $f := (f_0, \ldots, f_T)$ and $\cF := \cF_0 \times \cdots \times \cF_T$ with $\cF_T := \{ c(s_{1:T})\}$.

FVI visits states $s$ according to a sampling distribution $\beta$. $\beta$ should be able to obtain representative samples and exhaust all possible states in a certain sense, as outlined in Assumption \ref{assum:concentr}. Given a state $s_{0:t}$, we use the empirical distribution to calculate an approximation $\hat{V}_t(s_{0:t})$ to the value function $V_t(s_{0:t})$ in \eqref{eq:empi_OT}. We denote the set of couplings between empirical measures by $\Pi(\hat{\mu}^t_B, \hat{\nu}^t_B, s_{0:t}) := \Pi( \hat{\mu}_B(dx_{t+1}| x_{0:t}), \hat{\nu}_B(dy_{t+1}|y_{0:t}))$, where constant $B$ is the sample size. For notation convenience, we denote the Bellman operator as 
\begin{equation}
	\cT[f_{t+1}] (s_{0:t}) := \inf_{\pi(ds_{t+1}| s_{0:t}) \in \Pi( \mu^t, \nu^t, s_{0:t}) } \int f_{t+1} (s_{0:t}, s_{t+1}) \pi (ds_{t+1}| s_{0:t}).
\end{equation}
To approximate the value function at a given point, \eqref{eq:empi_OT} uses the empirical measures for marginals. Then we denote the empirical Bellman operator as
\begin{equation}
	\widehat{\cT}_B[f_{t+1}] (s_{0:t}) := \inf_{\pi(ds_{t+1}| s_{0:t}) \in \Pi( \hat{\mu}^t_B, \hat{\nu}^t_B, s_{0:t})} \int f_{t+1} (s_{0:t}, s_{t+1}) \pi (ds_{t+1}| s_{0:t}).
\end{equation}

When the mini-batch size $B$ is large, the classical empirical OT problem in \eqref{eq:empi_OT} incurs a high computational burden. Then we can consider entropic regularization and the Sinkhorn algorithm \citep{cuturi2013sinkhorn}. Denote the regularization parameter as $\varepsilon > 0$. We replace \eqref{eq:empi_OT} with
\begin{align*}
	& \hat{V}_{\varepsilon, t}(s^{n_t}_{0:t}) \\
	& \leftarrow \inf_{\pi(ds_{t+1} | s^{n_t}_{0:t}) \in \Pi(\hat{\mu}^t_B, \hat{\nu}^t_B, s^{n_t}_{0:t})}  \int \hat{f}_{t+1} (s^{n_t}_{0:t}, s_{t+1}) \pi(ds_{t+1} | s^{n_t}_{0:t})  \\
	& \hspace{4cm} + \varepsilon \text{KL} \left(\pi(ds_{t+1}| s^{n_t}_{0:t}) \Big\| \hat{\mu}_B(dx_{t+1} |x^{n_t}_{0:t} ) \otimes \hat{\nu}_B(dy_{t+1} | y^{n_t}_{0:t}) \right).
\end{align*}
$\text{KL}(\p \| \q) = \int \ln\left( \frac{d\p}{d\q}\right) d\p$ stands for the Kullback--Leibler divergence. Correspondingly, $\hat{V}_{t}(s^{n_t}_{0:t})$ in Algorithm \ref{algo:fvi} is replaced by $\hat{V}_{\varepsilon, t}(s^{n_t}_{0:t})$.

For later use, when entropic regularization is adopted, we denote the entropic Bellman operator and the empirical counterpart as
\begin{align*}
	& \cT_\varepsilon[f_{t+1}] (s_{0:t}) \\
	& := \inf_{\pi(ds_{t+1}| s_{0:t}) \in \Pi( \mu^t, \nu^t, s_{0:t}) } \int f_{t+1} (s_{0:t}, s_{t+1}) \pi (ds_{t+1}| s_{0:t}) \\
	& \hspace{4cm} + \varepsilon \text{KL} \left(\pi(ds_{t+1}| s_{0:t}) \Big\| \mu(dx_{t+1} |x_{0:t} ) \otimes \nu(dy_{t+1} | y_{0:t}) \right)
\end{align*}
and
\begin{align*} 
	& \widehat{\cT}_{\varepsilon, B} [f_{t+1}] (s_{0:t}) \\
	& := \inf_{\pi(ds_{t+1}| s_{0:t}) \in \Pi( \hat{\mu}^t_B, \hat{\nu}^t_B, s_{0:t})} \int f_{t+1} (s_{0:t}, s_{t+1}) \pi (ds_{t+1}| s_{0:t}) \\
	& \hspace{4cm} + \varepsilon \text{KL} \left(\pi(ds_{t+1}| s_{0:t}) \Big\| \hat{\mu}_B(dx_{t+1} |x_{0:t} ) \otimes \hat{\nu}_B(dy_{t+1} | y_{0:t}) \right).
\end{align*}

To find an approximator $\hat{f}_t$ of the value function $V_t$, we perform the empirical risk minimization (ERM) in \eqref{eq:emp_loss} with the squared loss function. \eqref{eq:emp_loss} uses the empirical estimation $\hat{V}_{t}(s^{n_t}_{0:t})$. The corresponding theoretical optimization problem for \eqref{eq:emp_loss} is
\begin{equation}\label{eq:trueRM} 
	\min_{f_t \in \cF_t} \int \big( f_t(s_{0:t}) - \cT[\hat{f}_{t+1}](s_{0:t}) \big)^2 \beta(ds_{1:t}| s_0).
\end{equation}
In this work, we always assume that the infima in \eqref{eq:emp_loss} and \eqref{eq:trueRM} are attained. This assumption is valid if the Weierstrass extreme value theorem \cite[Box 1.1]{santambrogio2015optimal} applies, which holds under relatively mild conditions. For neural networks with parameters in compact domains, this assumption is guaranteed. For later use, we denote the loss function as 
\begin{equation}
	\cL(f_t(s_{0:t}), \cT[f_{t+1}](s_{0:t})) :=  \big( f_t (s_{0:t}) - \cT[f_{t+1}](s_{0:t}) \big)^2.
\end{equation}

FVI introduces several inevitable errors, especially with continuous marginals $\mu$ and $\nu$. First, instead of utilizing the true sample distribution $\beta$ in \eqref{eq:emp_loss}, we can only sample a finite number of states $s^{n_t}$. This introduces the error associated with ERM, which can be further decomposed into approximation error and estimation error, as discussed in \citet[Section 5.2]{shalev2014understanding}. Second, the true value of $V(s_{0:t})$ is unknown and needs to be approximated by $\hat{V}_t(s_{0:t})$ in \eqref{eq:empi_OT}, which introduces additional errors. Third, ERM in \eqref{eq:emp_loss} relies on optimization methods like gradient descent to find an optimizer $\hat{f}_t$. However, the optimization algorithm may not converge to the exact optimizer and is typically terminated when a certain level of accuracy is reached. Consequently, this introduces optimization errors \cite[Section 2.2]{bottou2007tradeoffs}. In this work, we always assume the optimizer in the ERM \eqref{eq:emp_loss} can be obtained and thus ignore the optimization error. The FVI algorithm proceeds in a backward manner, and these errors propagate along the time dimension. In Section \ref{sec:sam_complex}, we delve into the analysis of sample complexity and establish error bounds.

\section{Sample complexity}\label{sec:sam_complex}
\subsection{Value suboptimality and Bellman error}
To evaluate the performance of FVI, we introduce the objective functional for a bicausal transport plan $\pi$ as follows:
\begin{align}
	J_0 (s_0; \pi) := \int c(s_{1:T}) \pi(ds_{1:T} | s_0).
\end{align}

Given an approximator $f = (f_0, ..., f_{T})$, we denote
$$\pi^f:= \bar{\pi}(ds_0) \pi^{f_{1}}(ds_{1}| s_{0}) \ldots \pi^{f_{t+1}}(ds_{t+1}| s_{0:t}) \ldots \pi^{f_{T}}(ds_{T}| s_{0:T-1})$$ 
as a bicausal transport plan implied by $f$, where $\pi^{f_{t+1}}(ds_{t+1}| s_{0:t})$ is obtained by
\begin{equation}
	\pi^{f_{t+1}}(ds_{t+1}| s_{0:t}) \in \arg \min_{\pi(ds_{t+1}| s_{0:t}) \in \Pi( \mu^t, \nu^t, s_{0:t}) } \int f_{t+1} (s_{0:t}, s_{t+1}) \pi (ds_{t+1}| s_{0:t}).
\end{equation}

Inspired by the concentrability coefficient in \cite{munos2003error}, we suppose that $\pi^{\hat{f}}$, implied by the approximator $\hat{f}$ from FVI, is dominated by the sampling distribution $\beta$ in the following sense:
\begin{assumption}\label{assum:concentr}
	Given an initial state $s_0$, for any time $t=0, ..., T-1$, suppose there exists a universal and finite constant $C \geq 1$ such that 
	\begin{equation}\label{eq:conc}
		\int_{\cX_{1:t} \times \cY_{1:t}} \left( \frac{\pi^{\hat{f}}(ds_{1:t} | s_0)}{\beta (ds_{1:t} | s_0)} \right)^2 \beta(ds_{1:t} | s_0) \leq C
	\end{equation}
	holds for the approximator $\hat{f}$ learned by the FVI algorithm. The $t=0$ case is interpreted as Dirac measures. The constant $C$ is called the concentrability coefficient.
\end{assumption}
When Assumption \ref{assum:concentr} is used in Lemma \ref{lem:subopt}, we only need to consider the bicausal coupling implied by the approximator $\hat{f}$ obtained by FVI, instead of all bicausal couplings in $\Pi_{bc}(\mu, \nu)$. The concentrability coefficient is widely used in RL with Markov decision processes (MDPs); see \citet[Equation 7]{munos2003error}, \citet[Assumption A.2]{munos2008finite}, \cite{fan2020theoretical,duan2021risk}. \cite{munos2008finite} compared it with the top-Lyapunov exponent of the MDP. \citet[Section 4]{chen2019information} argued the necessity of concentrability in the batch RL. 

We can also interpret Assumption \ref{assum:concentr} as a regularization on couplings. Given the concentrability coefficient $C$ in Assumption \ref{assum:concentr}, consider the following  bicausal OT problem with an additional constraint \eqref{eq:chi2}:
\begin{align}
	\inf_{\pi \in \Pi_{bc}(\mu, \nu, s_{0})} & \int c(s_{1:T}) \pi(ds_{1:T} | s_0), \label{eq:cons} \\
	\text{ subject to } & 	\int_{\cX_{1:t} \times \cY_{1:t}} \left( \frac{\pi(ds_{1:t} | s_0)}{\beta (ds_{1:t} | s_0)} \right)^2 \beta(ds_{1:t} | s_0) \leq C, \quad \forall \; t = 1, \ldots, T-1. \label{eq:chi2}
\end{align} 
Indeed, \eqref{eq:chi2} is the $\chi^2$-divergence $\chi^2 (\p \| \q) = \int \frac{d\p^2}{d\q} -1$. Both the KL-divergence and $\chi^2$-divergence belong to the $f$-divergence family. Since Lemma \ref{lem:subopt} below involves squared errors, the concentrability condition in \eqref{eq:conc} arises naturally. Additionally, the $\chi^2$-divergence facilitates sparse approximations of optimal transport couplings \cite{gonzalez2024quantitative,bayraktar2022stability}.

We denote $\pi^{*}_C$ as a minimizer of the problem \eqref{eq:cons}-\eqref{eq:chi2}. Note that $\pi^{*}_C$ may not be $\pi^f$ with some $f \in \cF$. Definition \ref{def:suboptimal} introduces the concept of value suboptimality to measure the accuracy of an approximator $f$ with respect to $\pi^*_C$. 
\begin{definition}\label{def:suboptimal}
	Given an initial state $s_0$, an approximator $f$ is said to be $\varepsilon$-optimal with respect to $\pi^*_C$ if 
	\begin{align*}
		& 2 \left|f_0(s_0) - \int f_1(s_{0:1}) \pi^f(ds_1|s_0)\right| + J_0 (s_0; \pi^f) - J_0 (s_0; \pi^*_C) \leq \varepsilon.
	\end{align*}
\end{definition}
Our definition of $\varepsilon$-optimality differs slightly from its counterpart in RL \cite[Definition 2.1]{duan2021risk}. In our case, $\pi^f$ depends only on $(f_1, \ldots, f_T)$ and not on $f_0$. Therefore, we include the first term to account for the suboptimality of $f_0$ in Definition \ref{def:suboptimal}. The constant $2$ is introduced for scaling purposes.

Value suboptimality in Definition \ref{def:suboptimal} can be bounded by another surrogate criterion known as the Bellman error, which is defined as follows:
\begin{equation}
	\cB(f) := \big( f_0(s_0) - \cT[f_1](s_0) \big)^2 + \sum^{T-1}_{t = 1} \int \big( f_t(s_{0:t}) - \cT[f_{t+1}](s_{0:t}) \big)^2 \beta(ds_{1:t}| s_0).
\end{equation}

The following result provides an upper bound on the suboptimality of the approximator $f$ in terms of Bellman errors. 
\begin{lemma}\label{lem:subopt}
	Given an initial state $s_0$, suppose Assumption \ref{assum:concentr} holds. For the approximator $\hat{f}$ learned by FVI, we have
	\begin{align*}
		& 2 \left|\hat{f}_0(s_0) - \int \hat{f}_1(s_{0:1}) \pi^{\hat{f}}(ds_1|s_0)\right| + J_0(s_0; \pi^{\hat{f}}) - J_0(s_0; \pi^*_C) \leq 2 \sqrt{CT\cB(\hat{f})},
	\end{align*}
	where $C$ is the concentrability coefficient in Assumption \ref{assum:concentr}.
\end{lemma}

Assumption \ref{assum:concentr} and Lemma \ref{lem:subopt} highlight the key considerations when selecting the sampling distribution $\beta$ and function class $\cF$. First, a typical choice for $\beta$ is the independent coupling $\beta = \mu \otimes \nu$. It is known that the optimal value $J_0(s_0; \pi^*_C)$ in the constrained problem \eqref{eq:cons} converges to the original optimal value $V_0(s_0)$ in \eqref{eq:V0} when $C \rightarrow \infty$; see \cite[Lemma 4.1 and Theorem 4.2]{eckstein2022comp}, where the proof extends to other divergences. The convergence rate in the bicausal setting remains an open problem. However, in the static setting, \cite{gonzalez2024quantitative} addresses the case of discrete marginals, and \cite{eckstein2024convergence} considers the general case using quantization. Second, the approximator $\hat{f}$ learned by the FVI algorithm should have a small Bellman error $\cB(\hat{f})$, ensuring that $C \cB(\hat{f})$ remains small. Later, with $\cF$ chosen appropriately, we show that $\cB(\hat{f}) \rightarrow 0$, as the sample size $B \rightarrow \infty$ and then $N \rightarrow \infty$. In summary, we first ensure that $\cB(\hat{f}) \rightarrow 0$, and then let $C  \rightarrow \infty$, such that $\hat{f}_0(s_0) \rightarrow V_0(s_0)$ by Lemma \ref{lem:subopt}.

As a side note, another approach to verify Assumption \ref{assum:concentr} and apply our results is to construct two sequences of {\it discrete} measures $\mu^n$ and $\nu^n$, $n = 1, 2, \ldots$, converging to $\mu$ and $\nu$ respectively in the adapted Wasserstein distance of order $p \geq 1$; see \cite[Theorem 3.6 and Lemma 4.1]{eckstein2022comp} for the exact procedure. Then, apply our algorithm to compute:
\begin{align*}
	V_0(s_0; \mu^n, \nu^n) := \inf_{\pi \in \Pi_{bc}(\mu^n, \nu^n, s_{0})}  \int c(s_{1:T}) \pi(ds_{1:T} | s_0).
\end{align*}
With $\beta = \mu^n \otimes \nu^n$, Assumption \ref{assum:concentr} holds, since  both $\beta$ and the domains $\cX$ and $\cY$ are discrete. Our sample complexity results still apply, and the proofs can be simplified. Finally, \cite[Theorem 3.6]{eckstein2022comp} yields $\lim_{n \rightarrow \infty} 	V_0(s_0; \mu^n, \nu^n)  = 	V_0(s_0)$.

In the following two subsections, we further establish bounds on Bellman errors $\cB(\hat{f})$ with tools from Rademacher complexity \citep{mohri2018foundations} and local Rademacher complexity \citep{bartlett2005local}. 
\subsection{Bellman error bounds by Rademacher complexity}
To quantify the learning capability of a function class, we recall the definition of (empirical) Rademacher complexity from \citet[Definitions 3.1 and 3.2, Chapter 3]{mohri2018foundations}:
\begin{definition}
	Given a function class $\cF$, the empirical Rademacher complexity of  $\cF$ with respect to a fixed sample $(s^1, ... , s^N)$ of size $N$ is defined as
	\begin{equation}\label{ERC}
		\E_\sigma \cR_N \cF := \E_\sigma \Big[ \sup_{f \in \cF} \frac{1}{N} \sum^N_{i=1} \sigma_i f(s^i) \Big],
	\end{equation}
	where $\sigma = (\sigma_1, ..., \sigma_N)$, and $\sigma_i$ are independent uniform random variables taking values in $\{ -1, + 1\}$.  $\sigma_i$ is referred to as Rademacher random variables.
	
	The Rademacher complexity of $\cF$ is the expectation of the empirical Rademacher complexity over all samples of size $N$ drawn according to $\beta$:
	\begin{equation}\label{RC}
		\E\cR_N \cF := \E_{S} [ \E_\sigma \cR_N \cF ].
	\end{equation}
\end{definition}

Typically, Rademacher complexity requires the functions to be bounded. Therefore, we impose Assumption \ref{assum:bound_K} as follows:
\begin{assumption}\label{assum:bound_K}
	Suppose that all $f_t \in \cF_t$, $t = 0, ..., T$, including the cost function $c(s_{1:T})$, are continuous. Moreover, there exists a constant $K \geq 0$ such that $\|f_t(s_{0:t}) \|_\infty \leq K$ for $t = 0, ..., T$.
\end{assumption}

We impose the approximate completeness assumption \ref{assum:complete} on function classes $\cF$, which will be verified later in Section \ref{sec:app_complete}. This assumption states that we can approximate the Bellman operator well at each time step, which is commonly adopted in RL; see \cite{munos2008finite,chen2019information,duan2021risk}.
\begin{assumption}\label{assum:complete}
	Fix an initial state $s_0$. For any time $t=0, ..., T-1$, there exists $\zeta > 0$ such that
	\begin{equation}
		\sup_{f_{t+1} \in \cF_{t+1}} \inf_{f_t \in \cF_t} \int \big( f_t(s_{0:t}) - \cT[f_{t+1}](s_{0:t}) \big)^2 \beta(ds_{1:t}| s_0) \leq \zeta.
	\end{equation}
\end{assumption}

Assumption \ref{assum:Lip} serves to quantify the difference between the true and empirical Bellman operators. If the functions considered are Lipschitz continuous, then Assumption \ref{assum:Lip} is valid by the duality of the Wasserstein distance $W_1$ and the fact that $W_1 \leq W_p,\, p \in [1, \infty)$. Another sufficient condition is provided by \citet[Lemma 3.5]{eckstein2022MA}.
\begin{assumption}\label{assum:Lip}
	Suppose the initial state $s_0$ is fixed. Given $t = 0, ..., T-1$, for any $f_{t+1} (s_{0:t+1}) \in \cF_{t+1}$, state $s_{0:t} \in \cX_{0:t} \times \cY_{0:t}$, coupling $\pi(ds_{t+1}| s_{0:t}) \in \Pi(\mu^t, \nu^t, s_{0:t})$, and coupling $\hat{\pi}(ds_{t+1}| s_{0:t}) \in \Pi(\hat{\mu}^t_B, \hat{\nu}^t_B, s_{0:t})$, there exists a universal constant $L > 0$, such that
	\begin{align*}
		& \Big| \int f_{t+1}(s_{0:t}, s_{t+1}) \pi(ds_{t+1}| s_{0:t}) - \int f_{t+1} (s_{0:t}, s_{t+1}) \hat{\pi}(ds_{t+1}|s_{0:t}) \Big| \\
		& \qquad \leq L W_p (\pi(ds_{t+1}| s_{0:t}), \hat{\pi}(ds_{t+1}| s_{0:t})).
	\end{align*}
\end{assumption}

Next, we consider a function that depends on the convergence of empirical measures for marginals:
\begin{align}
	\Delta(s_{0:t}; B) := & \big[ W^p_p(\mu(dx_{t+1}| x_{0:t}), \hat{\mu}_B(dx_{t+1}| x_{0:t})) \\
	& + W^p_p(\nu(dy_{t+1}| y_{0:t}), \hat{\nu}_B(dy_{t+1}| y_{0:t})) \big]^{1/p}. \nonumber
\end{align}

With Rademacher complexity, we can derive an upper bound for Bellman errors of the approximator learned through FVI.
\begin{lemma}\label{lem:RC}
	Assume the initial state $s_0$ is fixed. Suppose Assumptions \ref{assum:bound_K}, \ref{assum:complete}, and \ref{assum:Lip} hold. For any $\delta \in (0, 1)$, with probability at least $1 -\delta$ over the draw of an i.i.d. sample $\{ s^{n_t}_{0:t} \}$, where $n_t = 1, ..., N$, $t=0, ..., T-1$, the Bellman error of $\hat{f}$ obtained by FVI satisfies
	\begin{align*}
		\cB(\hat{f}) \leq & T \zeta + \frac{2 L^2}{N} \sum^{T-1}_{t=0} \sum^N_{n_t = 1} \Delta^2(s^{n_t}_{0:t}; B) + \frac{8 L K}{N} \sum^{T-1}_{t=0} \sum^N_{n_t = 1} \Delta(s^{n_t}_{0:t}; B) \\
		& + 8 K \sum^{T-1}_{t=0} \E \cR_N \cF_t  + 4 T K^2 \sqrt{\frac{2 \ln(T/\delta)}{N}}.
	\end{align*}
\end{lemma}

To provide the convergence rate of $\Delta(s^n_{0:t}; B)$, we present two results for the unregularized and entropic OT cases in FVI Algorithm \ref{algo:fvi}, respectively.

For the unregularized case, \cite{fournier2015rate} provided concentration inequalities for unconditional measures. In our problem, which involves a temporal structure, we assume additional knowledge about the dynamics of marginals in Assumption \ref{assum:light}. For instance, in the special case where $P_t = 0$ and $Q_t = 0$, Assumption \ref{assum:light} implies the existence of a universal constant $C_0$ that is independent of $x_{0:t}$ and $y_{0:t}$, satisfying the exponential moment condition stated in \citet[Theorem 2, Condition (1)]{fournier2015rate}. Furthermore, Assumption \ref{assum:light} can encompass more general scenarios, such as autoregressive (AR) models. If the domains are compact, Assumption \ref{assum:light} can be verified under mild conditions.

\begin{assumption}(Light-tailed distribution)\label{assum:light}
	For each time $t = 0,..., T-1$, there exist deterministic functions $P_t$ and $Q_t$, such that with given $x_{0:t} \in \cX_{0:t}$ and $y_{0:t} \in \cY_{0:t}$,
	\begin{align*}
		& \big(x_{t+1} -  P_t(x_{0:t}) \big) \sim \lambda(dx_{t+1} | x_{0:t}), \quad \big(y_{t+1} - Q_t(y_{0:t}) \big) \sim \eta(dy_{t+1} | y_{0:t}).
	\end{align*}
	And the distributions $\lambda$ and $\eta$ satisfy the following moment conditions with universal constants $a > p$, $\xi > 0$, $C_0 < \infty$, independent of $x_{0:t}$ and $y_{0:t}$:
	\begin{align*}
		\int_{\cX_{t+1}} e^{\xi |x_{t+1}|^a} \lambda(dx_{t+1}|x_{0:t}) \leq C_0, \quad 	\int_{\cY_{t+1}} e^{\xi |y_{t+1}|^a} \eta(dy_{t+1}|y_{0:t}) \leq C_0.
	\end{align*}
\end{assumption}

Next, we define a function for the rate of convergence as
\begin{equation}\label{eq:R}
	R(\delta, B) = \left\{
	\begin{array}{rll}
		& \left( \frac{\ln(C_1/\delta)}{C_2 B}\right)^{1/\max\{d/p, 2\}},  & B \geq  \frac{\ln(C_1/\delta)}{C_2}, \\
		& \left( \frac{\ln(C_1/\delta)}{C_2 B}\right)^{p/a}, & B \leq  \frac{\ln(C_1/\delta)}{C_2}.
	\end{array}\right.
\end{equation}
Here, constants $C_1$ and $C_2$ depend only on $p, d, a, \xi$, and $C_0$. These constants correspond to $C$ and $c$ in \citet[Theorem 2]{fournier2015rate}. Our assumption validates Condition (1) in \citet[Theorem 2]{fournier2015rate}.

With Assumption \ref{assum:light}, we can further characterize $\Delta(s^n_{0:t}; B)$ and provide the convergence rate in terms of $R(\delta, B)$. It is important to let the sample size $B \rightarrow \infty$ first, followed by $N \rightarrow \infty$.
\begin{theorem}\label{thm:light_RC}
	Assume the initial state $s_0$ is fixed. Suppose Assumptions \ref{assum:bound_K}, \ref{assum:complete}, \ref{assum:Lip},  and \ref{assum:light} hold. With probability at least $1 -\delta$, $\delta \in (0, 1)$, the Bellman error of $\hat{f}$ obtained by FVI satisfies
	\begin{align*}
		\cB(\hat{f}) \leq & T \zeta + 2 L^2 T \cdot 2^{2/p} R\left(\frac{\delta}{(2N+1)T}, B\right)^{2/p} + 8LKT \cdot 2^{1/p}R\left(\frac{\delta}{(2N+1)T}, B\right)^{1/p} \\
		& + 8 K \sum^{T-1}_{t=0} \E \cR_N \cF_t  + 4 T K^2 \sqrt{\frac{2 \ln((2N+1)T/\delta)}{N}}.
	\end{align*}
\end{theorem}

The curse of dimensionality is revealed in Theorem \ref{thm:light_RC} as $R(\delta, B)$ scales with $B^{-1/d}$. It is well-known that entropic regularization can improve sample complexity \citep{genevay2019sample}. To quantify the rate of convergence under entropic regularization, we define the function:
\begin{align}\label{eq:R_entr}
	R(\delta, B; \varepsilon) := \frac{C}{\sqrt{B}} \left( 1 + \max\{1, \varepsilon \} e^{C/\varepsilon}  \right) \left( \max\{1, \varepsilon^{-d/2}\} + \sqrt{\ln\left(1/\delta\right)} \right) + C \varepsilon \ln (C/\varepsilon).
\end{align}
Here, $C$ is a sufficiently large constant, which is calculated in the proof of Corollary \ref{cor:entropy_RC}. Importantly, $C$ is independent of the probability $\delta$, sample sizes $B$ and $N$, and regularization parameter $\varepsilon$. While $R(\delta, B; \varepsilon)$ scales with $1/\sqrt{B}$, it exhibits an exponential dependence on the regularization constant $\varepsilon$, and converges to infinity when $\varepsilon \rightarrow 0$. Furthermore, it is worth noting that Corollary \ref{cor:entropy_RC} imposes stronger assumptions on the cost function and domains compared to the previous results.

\begin{corollary}\label{cor:entropy_RC}
	Suppose $\cX_{0:T}$ and $\cY_{0:T}$ are compact. Assume the initial state $s_0$ is fixed. The functions $f_t \in \cF_{t}, t =0, ..., T$, including the cost function $c(s_{1:T})$, belong to $\cC^\infty$. Moreover, these functions are uniformly bounded by $K$ under Assumption \ref{assum:bound_K}. The derivatives of these functions up to a sufficiently high order are uniformly bounded by a universal constant $L' > 0$. Additionally, suppose Assumption \ref{assum:complete} holds. Then for a given $\delta \in (0, 1)$, with probability at least $1 - \delta$, the Bellman error of $\hat{f}$ obtained by FVI satisfies
	\begin{align*}
		\cB(\hat{f}) \leq & T \zeta + 2 T R\left(\frac{\delta}{(2N+1)T}, B; \varepsilon \right)^2 + 8 K T R\left(\frac{\delta}{(2N+1)T}, B; \varepsilon \right) \\
		& + 8 K \sum^{T-1}_{t=0} \E \cR_N \cF_t  + 4 T K^2 \sqrt{\frac{2 \ln((2N+1)T/\delta)}{N}}.
	\end{align*}
	The constant $C$ in $R(\delta, B; \varepsilon)$ depends on the constant $L'$, the uniform bound $K$, the diameters of $\cX$ and $\cY$, the time horizon $T$, and the dimension $d$. Importantly, it is independent of the probability $\delta$, sample sizes $B$ and $N$, and regularization parameter $\varepsilon$.
\end{corollary}

\subsection{Bellman error bounds by local Rademacher complexity}
The best rate for Rademacher complexity is typically on the order of $1/\sqrt{N}$. However, this rate may be suboptimal since Rademacher complexity considers global estimates of complexity for all functions in the class. In practice, algorithms may select functions from a smaller subset of the class under consideration. To address this, \cite{bartlett2005local} introduces the concept of local Rademacher complexity (LRC), which provides a faster rate of convergence. 
\begin{definition}\label{def:LRC}
	For a generic real-valued function class $\widetilde{\cF}$, consider a functional $\cM: \widetilde{\cF} \rightarrow \R^+$. Given a radius $r > 0$, the LRC is defined as
	\begin{equation*}
		\E\cR_N \{ f \in \widetilde{\cF} | \cM(f) \leq r\}.
	\end{equation*} 
\end{definition}
In the LRC, the functional $\cM$ is typically chosen as an upper bound on the variance of function $f$. The selection of the radius $r$ involves a trade-off between the size of the subset in the LRC and the complexity. \cite{bartlett2005local} demonstrated that the optimal choice of $r$ is related to the fixed point of certain functions that possess a sub-root property: 
\begin{definition}(\citet[Definition 3.1]{bartlett2005local})
	A function $\psi: [0, \infty) \rightarrow [0, \infty)$ is said to be {\it sub-root} if it is nonnegative, nondecreasing, and if $r \mapsto \psi(r)/\sqrt{r}$ is nonincreasing for $r > 0$. 
\end{definition}
According to \citet[Lemma 3.2]{bartlett2005local}, a sub-root function that is not the constant function $\psi \equiv 0$ has a unique positive fixed point denoted as $r^*$, which satisfies the equation $\psi(r^*) = r^*$. 

Regarding the LRC as a function of $r$, \cite{bartlett2005local} considered sub-root functions as upper bounds of the LRC. \citet[Theorem 3.3]{bartlett2005local} proved that the fixed point of a properly chosen sub-root function can provide sharper error bounds, as the LRC is smaller than the corresponding global averages.

We adopt the sub-root functions specified in Assumption \ref{assum:sub-root}. Note that the $t=0$ case uses Dirac measures at $s_0$.
\begin{assumption}\label{assum:sub-root}
	Suppose the initial state $s_0$ is fixed. Consider the FVI output $\hat{f}$. For each time point $t=0, ..., T-1$, define an optimizer $f^*_t$ under the distribution $\beta(ds_{1:t}|s_0)$ as
	\begin{equation*}
		f^*_t \in \arg\min_{f_t \in \cF_t}  \int | f_t(s_{0:t}) - \cT[\hat{f}_{t+1}] (s_{0:t}) |^2 \beta(ds_{1:t}|s_0).
	\end{equation*}
	Suppose there exists a sub-root function $\psi_t(r)$ such that
	\begin{align*}
		\psi_t(r) \geq  \E\cR_N \Big\{ f_t(s_{0:t}) - f^*_t(s_{0:t}) \Big| f_t \in \cF_t, \; \int \left(f_t(s_{0:t}) - f^*_t(s_{0:t}) \right)^2 \beta(ds_{1:t}| s_0) \leq r  \Big\}.
	\end{align*}
	Denote the unique fixed point of $\psi_t(r)$ as $r^*_t$.
\end{assumption}

Later in Proposition \ref{prop:radius_LRC}, we will demonstrate that the fixed point $r^*_t$ has an order of $\ln(N)/N$ for certain neural networks, which is faster than $1/\sqrt{N}$. Consequently, Lemma \ref{lem:LRC} can enhance the convergence rate with the LRC compared to Lemma \ref{lem:RC}. However, one drawback is that evaluating LRC is more challenging than Rademacher complexity.

\begin{lemma}\label{lem:LRC}
	Suppose the initial state $s_0$ is fixed and Assumptions \ref{assum:bound_K}, \ref{assum:complete}, \ref{assum:Lip}, and \ref{assum:sub-root} hold. For any $\theta > 1$ and every $\delta \in (0, 1)$, with probability at least $1 -\delta$ over the draw of an i.i.d. sample $\{ s^{n_t}_{0:t} \}$, $n_t = 1, ..., N$, $t=0, ..., T-1$, the Bellman error of $\hat{f}$ obtained by FVI satisfies
	\begin{align*}
		\cB(\hat{f}) \leq & T \zeta + \frac{\theta}{\theta -1} \Big( \frac{2 L^2}{N} \sum^{T-1}_{t=0} \sum^N_{n_t = 1} \Delta^2(s^{n_t}_{0:t}; B) + \frac{8 L K}{N} \sum^{T-1}_{t=0} \sum^N_{n_t = 1} \Delta(s^{n_t}_{0:t}; B)  \Big)\\
		& + 22528 \theta K^2 \sum^{T-1}_{t=0} r^*_t  +\frac{(88 + 832 \theta ) K^2 T}{N} \ln(T/\delta),
	\end{align*}
	where $r^*_t$ is the fixed point of $\psi_t$.
\end{lemma}

Although the constant $22528$ is explicit, it may not be optimal. However, the fixed point $r^*_t$, which depends on $N$, may play a more significant role than the constant $22528$. The explicit upper bounds on $r^*_t$ are derived in Proposition \ref{prop:cover_ineq} and \ref{prop:radius_LRC}.

Similar to the Rademacher complexity results in Theorem \ref{thm:light_RC} and Corollary \ref{cor:entropy_RC}, we can further bound $\Delta(s^{n_t}_{0:t}; B)$ with the light-tailed distribution assumption or stronger assumptions as in Corollary \ref{cor:entropy_RC}. We state the corresponding results in Theorem \ref{thm:LRC_light} and Corollary \ref{cor:LRC_entropy}. The proof is very similar and thus omitted.

\begin{theorem}\label{thm:LRC_light}
	Suppose the initial state $s_0$ is fixed and Assumptions \ref{assum:bound_K}, \ref{assum:complete}, \ref{assum:Lip}, \ref{assum:light}, and \ref{assum:sub-root} hold. For any $\theta > 1$ and every $\delta \in (0, 1)$, with probability at least $1 -\delta$, the Bellman error of $\hat{f}$ obtained by FVI satisfies
	\begin{align*}
		\cB(\hat{f}) \leq & T \zeta + \frac{\theta}{\theta -1} \Big(2 L^2 T \cdot 2^{2/p} R\left(\frac{\delta}{(2N+1)T}, B\right)^{2/p} + 8LKT \cdot 2^{1/p}R\left(\frac{\delta}{(2N+1)T}, B\right)^{1/p} \Big)\\
		& + 22528 \theta K^2 \sum^{T-1}_{t=0} r^*_t  +\frac{(88 + 832 \theta ) K^2 T}{N} \ln((2N+1)T/\delta),
	\end{align*}
	where $r^*_t$ is the fixed point of $\psi_t$.
\end{theorem}

\begin{corollary}\label{cor:LRC_entropy}
	Suppose $\cX_{0:T}$ and $\cY_{0:T}$ are compact. Assume the initial state $s_0$ is fixed. The functions $f_t \in \cF_{t}, t =0, ..., T$, including the cost function $c(s_{1:T})$, belong to $\cC^\infty$. Moreover, these functions are uniformly bounded by $K$ under Assumption \ref{assum:bound_K}. The derivatives of these functions up to a sufficiently high order are uniformly bounded by a universal constant $L' > 0$. Additionally, suppose Assumptions \ref{assum:complete} and \ref{assum:sub-root} also hold. Then for a given $\delta \in (0, 1)$, with probability at least $1 -\delta$, the Bellman error of $\hat{f}$ obtained by FVI satisfies
	\begin{align*}
		\cB(\hat{f}) \leq & T \zeta +  \frac{\theta}{\theta -1} \left( 2 T R\left(\frac{\delta}{(2N+1)T}, B; \varepsilon \right)^2 + 8 K T R\left(\frac{\delta}{(2N+1)T}, B; \varepsilon \right) \right) \\
		& + 22528 \theta K^2 \sum^{T-1}_{t=0} r^*_t  +\frac{(88 + 832 \theta ) K^2 T}{N} \ln((2N+1)T/\delta).
	\end{align*}
	The constant $C$ in $R(\delta, B; \varepsilon)$ from \eqref{eq:R_entr} depends on the constant $L'$, the uniform bound $K$, the diameters of $\cX$ and $\cY$, the time horizon $T$, and the dimension $d$. Importantly, it is independent of the probability $\delta$, sample sizes $B$ and $N$, and regularization parameter $\varepsilon$.
\end{corollary}

\section{Neural networks as function approximators}\label{sec:NN}
In this section, we validate the crucial assumptions when neural networks are adopted as approximators. We evaluate the fixed point $r^*_t$ of the sub-root function in Assumption \ref{assum:sub-root} using tools from the covering number. Furthermore, we verify the approximate completeness assumption by recent results of neural networks in H\"older function approximation \citep{schmidt2020,langer2021approximating}.


Denote the activation function as $\sigma(\cdot)$. For any given depth parameter $D$ and width parameters $\{ q_j\}^{D+1}_{j=0}$, a neural network function $f$ can be specified as
\begin{align}
	u \mapsto f(u) = W_D \sigma (W_{D-1} \sigma( W_{D-2} \cdots \sigma(W_1 \sigma (W_0 u + b_1) + b_2) \cdots b_{D-1}) + b_D) + b_{D+1}. \label{eq:NN}
\end{align}
Here, $W_j$ is a $q_{j+1} \times q_j$ weight matrix and $b_j \in \R^{q_j}$ is a shift (bias) vector. $q_0$ is the input dimension. In our setting, the output dimension $q_{D+1} = 1$.

\subsection{Calculation of local Rademacher complexity with covering number}
We utilize the covering number to evaluate the LRC. 

\begin{definition}\label{def:cover_no}
	Consider a metric space $(\cV, d)$ and a set $\cF \subseteq \cV$. For any $\varepsilon > 0$, a set $\cF^\Delta$ is called a proper $\varepsilon$-cover of $\cF$, if $\cF^\Delta \subseteq \cF$ and for every $f \in \cF$, there is an element $f^\Delta \in \cF^\Delta$ satisfying $d(f, f^\Delta) < \varepsilon$. The covering number $\cN(\varepsilon, \cF, d)$ is the cardinality of a minimal proper $\varepsilon$-cover of $\cF$:
	\begin{equation*}
		\cN(\varepsilon, \cF, d) := \min \{ |\cF^\Delta|: \cF^\Delta \subseteq \cF \text{ is a proper $\varepsilon$-cover of $\cF$} \}.
	\end{equation*}
	When $\cV$ is equipped with a norm $\| \cdot \|$, we denote by $\cN(\varepsilon, \cF, \| \cdot \|)$ the covering number of $\cF$ with respect to the metric $d(f, f') := \| f - f'\|$. 
\end{definition}
Our definition of covering numbers requires the $\varepsilon$-cover to belong to the original set, which is called proper \cite[p. 148]{anthony1999neural}. We provide some properties of covering numbers in the Appendix. As noted in \cite[p. 149]{anthony1999neural}, it is often more convenient to consider proper $\varepsilon$-covers. There is also a connection to improper $\varepsilon$-covers; see \cite[Lemma 10.6]{anthony1999neural}.

In statistical learning theory, we need to consider metrics endowed by samples. Let $\{s^1, ..., s^N \}$ represent $N$ sample paths, where the time subscripts are omitted for simplicity. We denote $\beta_N$ as the empirical measure supported on the given sample. For $p \in [1, \infty)$ and a function $f$, denote $\|f\|_{L_p(\beta_N)} = (\frac{1}{N}\sum^N_{i=1} |f(s^i)|^p)^{1/p}$ as the empirical norm. Set $\| f \|_{L_\infty(\beta_N)} = \max_{1\leq i \leq n} |f(s^i)|$. $\cN(\varepsilon, \cF, L_p(\beta_N))$ is the covering number of $\cF$ at scale $\varepsilon$ with respect to the $L_p(\beta_N)$ norm. Following \citet[Definition 3]{mendelson2003few}, we define 
\begin{equation*}
	\cN_p (\varepsilon, \cF) := \sup_N \sup_{\beta_N} \cN (\varepsilon, \cF, L_p (\beta_N)).
\end{equation*}
Note the difference between $\cN_p (\varepsilon, \cF)$ and $\cN(\varepsilon, \cF, \| \cdot \|_p)$. $\ln \cN_p (\varepsilon, \cF)$ is sometimes referred to as uniform entropic number \cite[Definition 3]{mendelson2003few}.

\citet[Theorem 14.5]{anthony1999neural} gave an upper bound on the covering number for neural networks with Lipschitz activation functions. They assumed that $\varepsilon$ is less than the uniform norm of the output in each layer times 2. These assumptions hold when using sigmoid activation functions, or ReLU functions with compact domains.  
\begin{proposition}(Informal)\label{prop:cover_ineq}
	Under the assumptions of \citet[Theorem 14.5]{anthony1999neural} on a class $\cF$ of neural networks with Lipschitz activation functions, for a small $\varepsilon > 0$, the covering number can be bounded by 
	\begin{equation}\label{eq:CoverNum}
		\ln \cN_\infty (\varepsilon, \cF) \leq C_{NN} (1 + \ln(1/\varepsilon)),
	\end{equation}
	where the constant $C_{NN}$ depends on the architecture parameters of the neural networks only.
\end{proposition}

\begin{proposition}\label{prop:radius_LRC}
	Consider any $t=0, ..., T-1$. Suppose Assumption \ref{assum:bound_K} holds. If the covering number bound \eqref{eq:CoverNum} holds for $\cF_t$, then the fixed point $r^*_t \lesssim \frac{\ln(N)}{N}$.
\end{proposition}

\subsection{Approximate completeness}\label{sec:app_complete}
Thanks to recent advancements in the convergence rates for approximating H\"older functions with neural networks \citep{schmidt2020,langer2021approximating}, we provide explicit bounds on the approximate completeness in Assumption \ref{assum:complete}, when the cost function and the approximators are $\alpha$-H\"older continuous with $\alpha \in (0, 1]$.
\subsubsection{H\"older space}
\begin{definition}\label{def:Holder}
	Given a H\"older smoothness index $\alpha \in (0, 1]$, denote the domain as $\cS \subseteq \R^{d'}$ for some positive integer $d'$. The H\"older space is given by 
	\begin{align*}
		\cH^\alpha (\cS, H) := \left\{ f: \cS \rightarrow \R : \|f\|_\infty +  	\sup_{\substack{s, s' \in \cS, \; s \neq s'}} \frac{|f(s) - f(s')|}{\|s - s'\|^\alpha_\infty} \leq H \right\}.
	\end{align*}
\end{definition}

Assumption \ref{assum:cond_Lip} is utilized to demonstrate that the Bellman operator on certain appropriate functions is H\"older continuous. Note that Assumption \ref{assum:cond_Lip} also applies to the cost function, since $f_T = c(s_{1:T})$. Additionally, Assumption \ref{assum:cond_Lip} is related to conditional kernels along different paths, while Assumption \ref{assum:Lip} applies to true and empirical measures.

\begin{assumption}\label{assum:cond_Lip}
	Suppose $s_0$ also varies in a closed subset $\cS_0 = \cX_0 \times \cY_0$. Given $t = 0, ..., T-1$, suppose for any $f_{t+1} \in \cF_{t+1}$, states $s_{0:t}, s'_{0:t} \in \cS_{0:t}$, coupling $\pi(ds_{t+1}| s_{0:t}) \in \Pi(\mu^t, \nu^t, s_{0:t})$, and coupling $\pi(ds_{t+1}| s'_{0:t}) \in \Pi(\mu^t, \nu^t, s'_{0:t})$, there exists a universal constant $L_{t+1} > 0$, such that
	\begin{align*}
		& \Big| \int f_{t+1}(s_{0:t}, s_{t+1}) \pi(ds_{t+1}| s_{0:t}) - \int f_{t+1} (s'_{0:t}, s_{t+1}) \pi(ds_{t+1}|s'_{0:t}) \Big|\\
		& \qquad \leq L_{t+1} W_p (\pi(ds_{t+1}| s_{0:t}), \pi(ds_{t+1}| s'_{0:t})) +  L_{t+1} \|s_{0:t} - s'_{0:t} \|^\alpha_\infty.
	\end{align*}
\end{assumption}

We introduce the following function, which characterizes the smoothness of conditional kernels of marginals:
\begin{equation}
	\Delta(s_{0:t}, s'_{0:t}) := \big[W^p_p(\mu(dx_{t+1}| x_{0:t}), \mu(dx_{t+1}| x'_{0:t})) + W^p_p(\nu(dy_{t+1}| y_{0:t}), \nu(dy_{t+1}| y'_{0:t})) \big]^{1/p}.
\end{equation}

Proposition \ref{prop:Holder} validates the H\"older smoothness of Bellman operators under Assumption \ref{assum:cond_Lip} and condition \ref{ineq:Holder}. 
\begin{proposition}\label{prop:Holder}
	Suppose $\| f_t \|_\infty \leq K_t$ for $t=1, ..., T$. Given any $t=0, ..., T-1$, suppose Assumption \ref{assum:cond_Lip} holds with some $L_{t+1} > 0$. Assume there exists a constant $H_t$ such that
	\begin{equation}\label{ineq:Holder}
		\sup_{s_{0:t}, s'_{0:t} \in \cS_{0:t}, \; s_{0:t} \neq s'_{0:t}} \frac{\Delta(s_{0:t}, s'_{0:t})}{\|s_{0:t} - s'_{0:t}\|^\alpha_\infty} \leq (H_t - K_{t+1} - L_{t+1})/L_{t+1}.
	\end{equation}
	Then $\cT[f_{t+1}] \in \cH^\alpha (\cS_{0:t}, H_t)$ for $t=0, ..., T-1$. 
\end{proposition}

In the following two subsections, we explore specific function classes that satisfy Assumptions \ref{assum:cond_Lip} and \ref{assum:complete}, including sparse ReLU networks and sigmoid networks. By appropriately choosing the neural network architecture, we can obtain explicit bounds on approximate completeness \citep{schmidt2020,langer2021approximating}.

\subsubsection{Sparse ReLU network}

Denote the ReLU activation function as $\sigma(u) := \max\{ u, 0\}$. We restrict to the class of ReLU networks where most parameters are zero.
\begin{definition}
	With a given depth $D$, width $\{ q_j\}^{D+1}_{j=0}$, sparsity $\gamma$, and uniform bound $K$, the sparse ReLU network is defined as 
	\begin{align*}
		\cF(D, \{ q_j\}^{D+1}_{j=0}, \gamma, K) := \{  f: &\; f \text{ specified as in \eqref{eq:NN}}, \quad \sigma(u) = \max\{ u, 0\}, \quad b_{D+1} = 0, \\
		&  \max_{j=0, ..., D} \| W_j \|_{\infty} \vee \| b_j \|_\infty \leq 1, \quad \sum^{D}_{j=0} \|W_j \|_0 + \| b_j\|_0  \leq \gamma, \\
		& \| f \|_\infty \leq K \},
	\end{align*}
	with the convention that $b_0 = 0$. 
\end{definition}

A more explicit upper bound on the covering number is also available for sparse ReLU networks. The proof is a slight extension of \citet[Lemma 3]{suzuki2019} and is thus omitted. Additionally, \citet[Lemma 5]{schmidt2020} also provided a similar estimate.
\begin{proposition}(\citet[Lemma 3]{suzuki2019})
	The covering number of sparse ReLU networks $\cF_t := \cF(D_t, \{ q^t_j\}^{D_t+1}_{j=0}, \gamma_t, K_t)$ can be bounded by
	\begin{equation*}
		\ln \cN_\infty(\delta, \cF_t) \leq \gamma_t \ln \left( (D_t+1)/\delta \prod^{D_t}_{j=0} (q^t_j + 1)^2  \right) \leq C_{ReLU} (1 + \ln(1/\delta)),
	\end{equation*}
	where the constant $C_{ReLU}$ is independent of $\delta$ and only depends on the depth $D_t$, width $\{ q^t_j\}^{D_t+1}_{j=0}$, and sparsity $\gamma_t$.
\end{proposition}

Thanks to \citet[Theorem 5]{schmidt2020}, we have the following result on approximate completeness, under the assumption that the domains are compact.
\begin{proposition}\label{prop:ReLU_complete}
	Suppose 
	\begin{itemize}
		\item[(1)] $\cX_t = \cY_t = [0, 1]^{d}$ for $t =0, 1, ..., T$;
		\item[(2)] the cost $\|c(s_{1:T})\|_\infty \leq K_T$ and satisfies Assumption \ref{assum:cond_Lip} with some $L_T>0$ and $\alpha \in (0, 1]$;
		\item[(3)] for $t=0, ..., T-1$, the function class $\cF_t = \cF(D_t, \{ q^t_j\}^{D_t+1}_{j=0}, \gamma_t, K_t)$ with some $D_t$, $\{ q^t_j\}^{D_t+1}_{j=0}$, $\gamma_t$, and $K_t$ specified below;
		\item[(4)] inequality \eqref{ineq:Holder} holds for large enough $H_t > 0$ at each $t=0, ..., T-1$.
	\end{itemize}
	Then 
	\begin{itemize}
		\item[(a)] $\cT[f_{t+1}] \in \cH^\alpha (\cS_{0:t}, H_t)$ for $t = 0, ..., T-1$;
		\item[(b)] ReLU network parameters are specified backwardly as follows: Suppose $\cF_{t+1}$ is chosen. Denote $d_t := (t+1)d$ for simplicity. For $\cF_{t}$, consider any integers $m_t \geq 1$ and $G_t \geq \max\{ (\alpha + 1)^{d_t}, (H_t + 1)e^{d_t} \}$, set width $\{ q^t_j\}^{D_t+1}_{j=0} = \{d_t, 6(d_t + 1)G_t, ..., 6(d_t + 1)G_t, 1\}$, depth $D_t = 8 + (m_t + 5)(1 + \lceil\log_2(d_t)\rceil)$, sparsity $\gamma_t \leq 141 (d_t + \alpha + 1)^{3+d_t} G_t (m_t + 6)$, and uniform bound $K_t = (2H_t + 1)(1+ d^2_t + \alpha^2) 6^{d_t} G_t 2^{-m_t} + H_t 3^\alpha G_t^{-\frac{\alpha}{d_t}} + K_{t+1}$;
		\item[(c)] at time $t=0, ..., T-1$, the approximate completeness assumption \ref{assum:complete} holds with $\zeta_t = (K_t - K_{t+1})^2$.
	\end{itemize}
\end{proposition}

\subsubsection{Sigmoid neural networks}
Since Corollaries \ref{cor:entropy_RC} and \ref{cor:LRC_entropy} regarding the sample complexity with entropic regularization require a smooth function class, we provide the corresponding result for neural networks with the sigmoid activation function, defined as follows:
\begin{equation*}
	\sigma(u) = \frac{1}{1+ e^{-u}}.
\end{equation*}

\begin{definition}\label{def:sig_NN}
	With a given depth $D$, width $q$, and parameter uniform bound $U$, the sigmoid neural network is defined as 
	\begin{align*}
		\Sigma(D, q, U, K) := \{  f: &\; f \text{ specified as in \eqref{eq:NN}}, \quad \sigma(u) = \frac{1}{1+ e^{-u}}, \quad q_{D} = \cdots = q_1 = q,\\
		& \quad \max_{j=0, ..., D} \| W_j \|_{\infty} \vee \| b_{j+1} \|_\infty \leq U, \quad \| f \|_\infty \leq K  \}.
	\end{align*} 
\end{definition}

Definition \ref{def:sig_NN} adheres to the specification in \cite{langer2021approximating} and does not impose any sparsity constraint. By utilizing \citet[Theorem 1]{langer2021approximating}, we establish the approximate completeness result for sigmoid neural networks.

\begin{proposition}\label{prop:sigmoid}
	Suppose 
	\begin{itemize}
		\item[(1)] $\cX_t = \cY_t = [-a, a]^{d}$ with some $a\in[1, \infty)$, for all $t =0, 1, ..., T$;
		\item[(2)] the cost $\|c(s_{1:T})\|_\infty \leq K_T$ and satisfies Assumption \ref{assum:cond_Lip} with some $L_T>0$ and $\alpha \in (0, 1]$;
		\item[(3)] for $t=0, ..., T-1$, the function class $\cF_t = \Sigma(D_t, q_t, U_t, K_t)$ with some $D_t$, $q_t$, $U_t$, and $K_t$ specified below;
		\item[(4)] inequality \eqref{ineq:Holder} holds for large enough $H_t > 0$ at each $t=0, ..., T-1$.
	\end{itemize}
	Then 
	\begin{itemize}
		\item[(a)] $\cT[f_{t+1}] \in \cH^\alpha (\cS_{0:t}, H_t)$ for $t = 0, ..., T-1$;
		\item[(b)] sigmoid network parameters are specified backwardly as follows: Suppose $\cF_{t+1}$ is chosen. Denote constant $d_t := (t+1)d$ for simplicity. For $\cF_{t}$, consider any integers $M_t$ satisfying $M^{2\alpha}_t \geq \max\{ 2c_{2,t} (\max\{a, K_{t+1} \})^3, c_{3,t}, 2^{d_t}, 12 d_t\}$, for some constant $c_{2, t}$ and $c_{3, t}$. Set depth $D_t = 8 + \lceil\log_2(d_t)\rceil$, width $q_t = 2^{d_t} (2M^{d_t}_t(1+d_t)^2 + 2M^{d_t}_t (1+ d_t) + 13 d_t )$, and uniform parameter bound $ U_t = c_{4,t} (\max\{a, K_{t+1}\})^{11} e^{6 \times 2^{2(d_t + 1)+1}ad_t} M^{16\alpha+2d_t+9}_t$. Moreover, let $K_t : = K_{t+1} + \frac{c_{5, t} (\max\{a, K_{t+1}\})^3}{M^{2\alpha}_t}$ for some constant $c_{5, t} > 0$;
		\item[(c)] at time $t=0, ..., T-1$, the approximate completeness assumption \ref{assum:complete} holds with $\zeta_t := (K_t - K_{t+1})^2$.
	\end{itemize}
\end{proposition}

\section{Numerical results}\label{sec:numerics}
\subsection{Gaussian data}
It is rare for OT problems to have explicit solutions, even in the single-period case. One exception is the Gaussian distribution. Hence, for the numerical test, we consider $\mu$ and $\nu$ as Gaussian distributions with linear dynamics: 
\begin{equation}\label{eq:Gaussian}
	\begin{aligned}
		x_{t+1} & = x_t + \lambda_t, \quad \lambda_t \sim N(0, \Sigma_x), \\
		y_{t+1} & = y_t + \eta_t, \quad \eta_t \sim N(0, \Sigma_y).
	\end{aligned}
\end{equation}
Here, $x_t \in \R^d$ and $y_t \in \R^d$ are two AR(1) processes in the context of time-series analysis. The white Gaussian noise processes $\lambda_t$ and $\eta_t$ are independent, with mean zero and constant covariance matrices $\Sigma_x$ and $\Sigma_y$, respectively. The bicausal OT problem with quadratic cost is defined as follows:
\begin{align*}
	\inf_{\pi \in \Pi_{bc}(\mu, \nu, x_0, y_0)}  \int \sum^T_{t=1} (x_t - y_t)^2\pi(dx_{1:T}, dy_{1:T} | x_0, y_0).
\end{align*}
The explicit solution to this problem is given by 
\begin{align*}
	& T |x_0 - y_0|^2 + \frac{T(T+1)}{2} \left( \tr[\Sigma_x] + \tr[\Sigma_y] - 2 \tr \left[\sqrt{\sqrt{\Sigma_x} \Sigma_y \sqrt{\Sigma_x} } \right] \right).
\end{align*} 
This result can be proven similarly to \citet[Corollary 2.2]{han2023dist} by applying the single-period case \citep{givens1984class} repeatedly. Additionally, it is worth noting that the single-period case is closely related to the Bures metric.

In the experiment, we consider an agent that lacks knowledge of the actual underlying distribution but can sample data from it. From a theoretical standpoint, Gaussian data do not have compact support, which violates the assumptions in Propositions \ref{prop:ReLU_complete} and \ref{prop:sigmoid}. However, the FVI algorithm is still applicable, and the empirical distribution has bounded support. To assess the effectiveness of our method, we compare it with the LP and adapted Sinkhorn algorithms \citep{eckstein2022comp}. All experiments were performed on a laptop equipped with an Intel Core i5-12500H 3.10GHz CPU and an NVIDIA RTX 3050 Ti Laptop GPU. The PyTorch library was utilized for neural network implementation.

\subsection{One-dimensional data}
We begin by examining the one-dimensional case ($d=1$). The variances are fixed as $\Sigma_x = 1.0$ and $\Sigma_y = 0.5^2$. The initial values are set as $x_0 = 1.0$ and $y_0 = 2.0$.

\subsubsection{Linear programming method}
Both the LP and adapted Sinkhorn methods require a discrete distribution to approximate the continuous distribution $\mu$ and $\nu$. To accomplish this, we employ a non-recombining binomial tree construction as follows. Initially, at time $t=1$, we simulate $x_1$ and partition the samples into two sets. We then use the mean and frequency of each set to create two nodes in the tree. Next, we sample $x_2$ conditioned on each $x_1$. At time 2, there will be four nodes in total, as the middle two nodes are almost surely not recombined. We repeat this procedure until time $T$. Both the LP and adapted Sinkhorn algorithms consider this tree as the underlying discrete distribution. Figure \ref{fig:tree} illustrates an example of a non-recombining binomial tree. Since the agent does not possess knowledge of the true distribution, constructing a recombining binomial tree to reduce the number of nodes becomes challenging. Additionally, we also explored the adapted empirical measures with $k$-means as proposed in \cite{backhoff2020estimating,eckstein2022comp}. However, since the number of cubes after the partition is on the order of $N^{1/(T+1)}$ \cite[Definition 1.2]{backhoff2020estimating}, it requires a large sample size $N$ even for two cubes under long time horizons. Hence, we consider the non-recombining tree since it is more applicable.
\begin{figure}[h]
	\centering
	\includegraphics[width=0.15\linewidth]{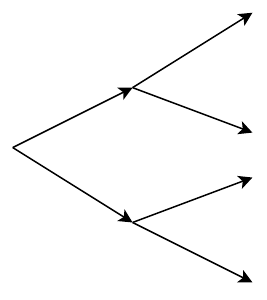}
	\caption{A non-recombining binomial tree.}\label{fig:tree}
\end{figure}


\begin{table}[h]
	\centering
	\caption{The LP method with backward induction for bicausal OT.}\label{tab:backLP}
	\begin{tabular}{cccc}
		\hline
		Horizon $T$ & Actual & Estimated mean (SD) & Average runtime  \\
		\hline 
		1 & 1.25 & 1.217 (0.102) & 0.102 \\
		2 & 2.75 & 2.734 (0.243) & 0.270  \\
		3 & 4.5 & 4.262 (0.488) & 0.587 \\
		4 & 6.5 & 6.295 (0.401) & 1.311 \\
		5 & 8.75 & 8.416 (0.600) & 2.909 \\
		6 & 11.25 & 10.414 (1.028) & 6.201 \\
		7 & 14.0 & 12.713 (1.071) & 13.733 \\
		8 & 17.0 & 16.114 (0.902)  & 32.178 \\
		9 & 20.25 &  19.051 (1.417)  & 83.298 \\
		10 & 23.75 & 22.649 (1.277)  &  248.249\\
		11 &   27.5  &  25.557 (1.646) & 794.041 \\
		\hline
	\end{tabular}
\end{table}

For the LP and adapted Sinkhorn algorithms, we utilize the implementation by \cite{eckstein2022comp}, which can be found at \url{https://github.com/stephaneckstein/aotnumerics}. The LP formulation can be implemented using backward induction or directly as described in  \citet[Lemma 3.11]{eckstein2022comp}. We first test the LP formulation in \citet[Lemma 3.11]{eckstein2022comp}, employing the Gurobi solver with an academic license. In the case of $T=10$, we observed that this method took longer than 1000 seconds to finish. Gurobi reported the model had 701098 rows, 1048576 columns, and 23068672 nonzeros, even after the presolve phase to reduce the problem size. Consequently, we resorted to using the LP with backward induction, which solves an OT problem of shape $2 \times 2$ at each time and state. The estimated results are presented in Table \ref{tab:backLP}. For each configuration, we conducted 10 instances to compute the mean and standard deviation (SD). The average runtime, measured in seconds per instance, includes the data sampling time as well. Table \ref{tab:backLP} demonstrates that the non-recombining tree can effectively approximate the AR(1) model, yielding accurate means and low variances. However, the runtime grows rapidly, indicating that the LP method does not scale well. We also conducted tests for the case of $T=12$ and found that the runtime was approximately 2600 seconds.

\subsubsection{Adapted Sinkhorn method}

Table \ref{tab:sink} demonstrates the remarkable accuracy of the adapted Sinkhorn method. The stopping criterion is when the change in dual values is less than $10^{-4}$. The actual values presented in Tables \ref{tab:backLP} and \ref{tab:sink} correspond to continuous distributions. Therefore, the errors between the actual and estimated values also include the approximation errors resulting from discretization with binomial trees. 

We observe that the runtime of the adapted Sinkhorn method increases rapidly as the time horizon expands. Interestingly, the adapted Sinkhorn method is even slower than the LP method. This can be attributed to the fact that, in order to obtain the potential functions, the adapted Sinkhorn method \cite[Lemma 6.6]{eckstein2022comp} needs to iterate over all possible scenarios of $(x_t, y_t)$ for each time $t$. In a non-recombining binomial tree model, the number of these scenarios grows exponentially with the horizon $T$. Our experimental results do not contradict the findings of \cite{eckstein2022comp}. When the time horizon is short (as in $T=2$ in \cite{eckstein2022comp}), this drawback of the adapted Sinkhorn algorithm is mitigated, and it can be faster if the number of possible scenarios is moderate, but not excessively small.

Moreover, although the entropic OT converges to the unregularized OT when $\varepsilon \rightarrow 0$, the Sinkhorn algorithm can encounter numerical instability issues when $\varepsilon$ is small. As the time horizon $T$ increases, a larger value of $\varepsilon$ becomes necessary. This introduces a trade-off between stability, accuracy, and speed by varying $\varepsilon$. Compared with the FVI results in Table \ref{tab:FVI_1d}, we observe that for similar levels of accuracy, the adapted Sinkhorn method is much slower for longer time horizons. 
Additionally, when $T \geq 20$, both the LP and adapted Sinkhorn methods fail to converge within a reasonable time.

\begin{table}[h]
	\centering
	\caption{The adapted Sinkhorn method for bicausal OT.}\label{tab:sink}
	\begin{tabular}{ccccc}
		\hline
		Horizon $T$ & Actual & Estimated mean (SD) & Average runtime & $\varepsilon$ in entropic regularization  \\
		\hline 
		1 & 1.25 & 1.201 (0.103) & 0.087 & 0.1 \\
		2 & 2.75 & 2.701 (0.251) & 0.313  & 0.1 \\
		3 & 4.5 & 4.215 (0.488) & 0.815 & 0.1 \\
		4 & 6.5 & 6.229 (0.389) & 2.320 & 0.1 \\
		5 & 8.75 & 8.297 (0.656) & 8.067 & 0.1 \\
		6 &  11.25  & 10.332 (1.014) & 22.600 & 0.2 \\
		7 &   14.0   &  12.662 (1.079) & 39.994 & 0.4 \\
		8 & 17.0 & 16.085 (0.915)  & 205.034 & 0.6 \\
		9 & 20.25 & 19.132 (1.416) & 517.958 & 0.8 \\
		10 & 23.75 & 22.874 (1.297) & 1443.557 & 1.0 \\
		\hline
	\end{tabular}
\end{table}

\subsubsection{The FVI algorithm}
In the FVI algorithm, we define the sampling distribution as $\beta = \mu \otimes \nu$. Suppose the agent knows that the true value function is separable in time and state variables. These structural properties are neglected in the discretization methods, while the FVI method can incorporate these properties in the design of approximators. We use a neural network class $\mathcal{F}_t$ with the same architecture across different time steps. The architecture is specified as $F_1(T - t) F_2(x_t, y_t) + F_3(T - t)$, where $F_1, F_2$ and $F_3$ are learnable approximators. In particular, $F_1$ and $F_3$ are quadratic functions of $T- t$, while $F_2$ is a multilayer ReLU neural network with a depth $D=2$ and hidden width $q_j = 8$. We have not enforced sparsity in the implementation. Since $F_1$, $F_2$, and $F_3$ are not identifiable, we introduce a sigmoid function to scale the last output of $F_2$ into the range $(0, 1)$. The network is initialized using the default method in PyTorch. Our code is available at \url{https://github.com/hanbingyan/FVIOT}.

For optimization, we utilize the Adam optimizer with a learning rate of $0.01$. Gradients and parameters are truncated elementwise to the range of $[-1, 1]$ in each gradient descent step. Through experiments, we observe that the smooth $L_1$ loss in PyTorch outperforms the squared loss function, possibly because it is less affected by outliers. The smooth $L_1$ loss function, closely related to the Huber loss, is defined as follows:
\begin{equation*}
	l(f, v) = \left\{
	\begin{array}{rll}
		&  \frac{1}{2 \tau} (f - v)^2, & \text{ if } | f - v | < \tau, \\
		& | f - v| - \tau/2, & \text{ otherwise. }
	\end{array}\right.
\end{equation*}
The default threshold value is $\tau = 1.0$.

Table \ref{tab:FVI_1d} presents the estimation results for the FVI algorithm with the unregularized empirical OT. To prevent overfitting, we reduce the number of gradient steps ($G$) as the time horizon increases. Other hyperparameters are fixed, including a batch size of $128$ for gradient descent, $2000$ sample paths ($N$), and a sample size of $50$ for the empirical OT ($B$). Although the estimation error is higher compared to the LP and adapted Sinkhorn methods, it remains within moderate and acceptable levels. As expected, the variance increases with the horizon. The reported runtime includes data sampling. For small horizons, our method requires more time. However, the computational burden increases linearly. The FVI algorithm becomes faster than the LP and adapted Sinkhorn methods for longer horizons. Based on these observations, we assert that the FVI method provides acceptable accuracy and exhibits scalability.

\begin{table}[h]
	\centering
	\caption{FVI with one-dimensional data. }\label{tab:FVI_1d}
	\begin{tabular}{ccccc}
		\hline
		Horizon $T$ & Actual & Estimated mean (SD) & Average runtime & Gradient steps $G$ \\
		\hline 
		1 & 1.25 & 1.311 (0.049) & 10.592 &  50 \\
		2 & 2.75 & 2.932 (0.181) & 15.655  & 50 \\
		3 & 4.5 & 5.031 (0.306) & 22.184 & 50 \\
		4 & 6.5 & 7.261 (0.534) & 28.713 & 50 \\
		5 & 8.75 & 9.710 (0.951) & 35.560 & 50 \\
		6 &  11.25  & 12.315 (1.171) & 41.256 & 40 \\
		7 &   14.0   &  14.901 (1.354) & 47.410 & 30 \\
		8 & 17.0 &  17.907 (2.660) & 55.071 & 20 \\
		9 & 20.25 &  21.392 (2.426) & 60.681 & 20\\
		10 & 23.75 & 24.455 (3.610) & 66.435 & 20 \\
		20 & 72.5 & 72.020 (8.172) & 131.445 & 20 \\
		40 & 245.0 & 235.266 (38.341) & 257.693 & 20 \\
		\hline
	\end{tabular}
\end{table}

\subsection{Multidimensional data}

\begin{table}[H]
	\centering
	\caption{FVI with multidimensional data.} \label{tab:multi}
	\begin{tabular}{ccccc}
		\hline
		Horizon $T$ & Dimension $d$ & Actual & Estimate mean (SD) & Average runtime \\
		\hline 
		$T=5$ & $d=5$ & 100.0 & 99.861 (3.421)  & 336.310  \\
		$T=5$ & $d=10$ & 200.0 & 191.948 (6.900) & 358.793 \\
		$T=5$ & $d=15$ & 300.0 & 268.245 (13.931) & 336.887 \\
		$T=5$ & $d=20$ & 400.0 & 351.073 (7.751) & 334.874 \\
		\hline
	\end{tabular}
\end{table}

The implementations of the LP and adapted Sinkhorn algorithms in \cite{eckstein2022comp} do not account for multidimensional data. However, our FVI method offers a uniform framework for handling such data. To illustrate, we assume that the covariance matrices are diagonal. Consider $x_0 = \mathbf{1}_d$, $y_0 = 2 \mathbf{1}_d$, $\Sigma_x = 1.1^2 I_d$, and $\Sigma_y = 0.1^2 I_d$. We set the time horizon $T=5$.

The empirical OT suffers from the curse of spatial dimensionality, necessitating larger sample sizes ($B$) and a greater number of paths ($N$). In our experiment, we set these hyperparameters to $N=4000$, $B=300$, and $G=400$. Since the actual values are large, we do not impose any parameter constraints in this particular experiment. The results are summarized in Table \ref{tab:multi}. We observe accurate estimation outcomes for low-dimensional cases with $d=5$ and $d=10$. As anticipated, estimates become less accurate as the dimension $d$ increases. It is worth noting that the runtime of our method solely relies on the hyperparameters and remains unaffected by the input dimension $d$.

\section{Concluding remarks}

In summary, our FVI algorithm is well-suited for tackling problems with long time horizons and extensive state spaces. However, if the horizon is short and the state space is small, the LP and Sinkhorn algorithms are more appropriate.

One drawback of the FVI algorithm is its tendency to exhibit larger variance and sensitivity to data with high volatility. While the discretization error may be more pronounced when dealing with a volatile true distribution, the estimated values produced by the FVI algorithm deteriorate more severely compared to the Sinkhorn algorithm. This disparity could potentially be attributed to the ERM framework employed by the FVI. A future direction is to mitigate the variance in the FVI estimations.


\appendix

\section{Proofs of results in Section \ref{sec:sam_complex}}

\begin{proof}[Proof of Lemma \ref{lem:subopt}]
	To obtain an upper bound, we note that
	\begin{equation}
		J_0(s_0; \pi^{\hat{f}}) - J_0(s_0; \pi^*_C) = J_0(s_0; \pi^{\hat{f}}) - \hat{f}_0 (s_0) + \hat{f}_0 (s_0) - J_0(s_0; \pi^*_C).
	\end{equation}
	We decompose $f_0(s_0) - J_0(s_0; \pi)$ for a generic approximator $f$ and bicausal transport $\pi$. Using the definition of the objective functional $J$, we have
	\begin{align*}
		& f_0 (s_0) - J_0(s_0; \pi) \\
		& \qquad =  f_0 (s_0) - \int c(s_{1:T}) \pi(ds_{1:T}|s_0) \\
		& \qquad = \int \big[ f_0 (s_0) - f_1(s_{0:1}) + f_1(s_{0:1}) - \cdots - f_{T-1}(s_{0:T-1}) + f_{T-1}(s_{0:T-1}) \\
		& \qquad \qquad \quad - c(s_{1:T}) \big] \pi(ds_{1:T}|s_0) \\
		& \qquad = f_0(s_0) - \int f_1(s_{0:1}) \pi(ds_1|s_0) + \int \Big[ f_1(s_{0:1}) - \int f_2(s_{0:2}) \pi(ds_2|s_{0:1}) \Big] \pi(ds_1|s_0) + \\
		& \qquad \quad \cdots + \int \Big[ f_t(s_{0:t}) - \int f_{t+1}(s_{0:t+1}) \pi(ds_{t+1}|s_{0:t}) \Big] \pi(ds_{1:t}|s_0) + \\
		& \qquad \quad \cdots +  \int \Big[ f_{T-1}(s_{0:T-1}) - \int c(s_{1:T}) \pi(ds_T|s_{0:T-1}) \Big] \pi(ds_{1:T-1}|s_0).
	\end{align*}
	Since $\pi^{\hat{f}}$ is optimal with $\hat{f}$, we have
	\begin{equation*}
		\int \hat{f}_{t+1} (s_{0:t}, s_{t+1}) \pi^{\hat{f}} (ds_{t+1}| s_{0:t}) = \cT[\hat{f}_{t+1}] (s_{0:t}).
	\end{equation*}
	Thus,
	\begin{align*}
		\hat{f}_0 (s_0) - J_0(s_0; \pi^{\hat{f}}) = \sum^{T-1}_{t=0}  \int \big\{ \hat{f}_t(s_{0:t}) - \cT[\hat{f}_{t+1}] (s_{0:t}) \big\} \pi^{\hat{f}}(ds_{1:t}|s_0).
	\end{align*}
	
	Next, we observe that $\pi^*_C$ is sub-optimal for $\hat{f}_t$. Then
	\begin{align*}
		\int \hat{f}_{t+1} (s_{0:t}, s_{t+1}) \pi^*_C (ds_{t+1}| s_{0:t}) \geq \cT[\hat{f}_{t+1}] (s_{0:t}).
	\end{align*}  
	It yields
	\begin{align*}
		\hat{f}_0 (s_0) - J_0(s_0; \pi^*_C) & = \sum^{T-1}_{t=0} \int \Big[ \hat{f}_t(s_{0:t}) - \int \hat{f}_{t+1}(s_{0:t+1}) \pi^*_C(ds_{t+1}|s_{0:t}) \Big] \pi^*_C(ds_{1:t}|s_0) \\
		& \leq \sum^{T-1}_{t=0} \int \Big[ \hat{f}_t(s_{0:t}) - \cT[\hat{f}_{t+1}] (s_{0:t}) \Big] \pi^*_C(ds_{1:t}|s_0).
	\end{align*}
	We combine these two inequalities and obtain
	\begin{align*}
		& J_0(s_0; \pi^{\hat{f}}) - J_0(s_0; \pi^*_C) = J_0(s_0; \pi^{\hat{f}}) - \hat{f}_0 (s_0) + \hat{f}_0 (s_0) - J_0(s_0; \pi^*_C) \\
		& \qquad \leq - \sum^{T-1}_{t=0}  \int \big\{ \hat{f}_t(s_{0:t}) - \cT[\hat{f}_{t+1}] (s_{0:t}) \big\} \pi^{\hat{f}}(ds_{1:t}|s_0) \\
		& \qquad \quad + \sum^{T-1}_{t=0} \int \Big[ \hat{f}_t(s_{0:t}) - \cT[\hat{f}_{t+1}] (s_{0:t}) \Big] \pi^*_C(ds_{1:t}|s_0) \\
		& \qquad  = \sum^{T-1}_{t=1} \int \Big[ \hat{f}_t(s_{0:t}) - \cT[\hat{f}_{t+1}] (s_{0:t}) \Big] \pi^*_C(ds_{1:t}|s_0) \\
		& \qquad \quad - \sum^{T-1}_{t=1}  \int \big\{ \hat{f}_t(s_{0:t}) - \cT[\hat{f}_{t+1}] (s_{0:t}) \big\} \pi^{\hat{f}}(ds_{1:t}|s_0).
	\end{align*}
	By Cauchy-Schwarz inequality and Assumption \ref{assum:concentr}, 
	\begin{align*}
		& \Big| \sum^{T-1}_{t=1}  \int \big\{ \hat{f}_t(s_{0:t}) - \cT[\hat{f}_{t+1}] (s_{0:t}) \big\} \pi^{\hat{f}}(ds_{1:t}|s_0)  \Big| + \left|\hat{f}_0(s_0) - \int \hat{f}_1(s_{0:1}) \pi^{\hat{f}}(ds_1|s_0)\right|\\
		& \qquad \leq \sum^{T-1}_{t=0} \Big( \int | \hat{f}_t(s_{0:t}) - \cT[\hat{f}_{t+1}] (s_{0:t}) |^2 \beta(ds_{1:t}|s_0) \Big)^{1/2} \Big( \int \left( \frac{\pi^{\hat{f}}(ds_{1:t} | s_0)}{\beta (ds_{1:t} | s_0)} \right)^2 \beta(ds_{1:t} | s_0) \Big)^{1/2} \\
		&\qquad  \leq \sqrt{C} \sum^{T-1}_{t=0} \Big( \int | \hat{f}_t(s_{0:t}) - \cT[\hat{f}_{t+1}] (s_{0:t}) |^2 \beta(ds_{1:t}|s_0) \Big)^{1/2} \\
		& \qquad \leq \sqrt{CT}  \sqrt{ \sum^{T-1}_{t=0} \int | \hat{f}_t(s_{0:t}) - \cT[\hat{f}_{t+1}] (s_{0:t}) |^2 \beta(ds_{1:t}|s_0)},
	\end{align*} 
	where we also used the inequality $\sum^{T-1}_{t=0} \sqrt{a_t} \leq \sqrt{T} \sqrt{\sum^{T-1}_{t=0} a_t}$ for $a_t \geq 0$.
	
	The second term with $\pi^*_C$ can be bounded similarly by noting the condition \eqref{eq:chi2}. By the definition of Bellman error, we obtain the desired result. 
\end{proof}

\begin{proof}[Proof of Lemma \ref{lem:RC}]
	We first consider the Bellman error at a generic time $t$. Suppose $\hat{f}_{t+1}$ is obtained. Define an optimizer $f^*_t$ under the distribution $\beta(ds_{1:t}|s_0)$ as
	\begin{equation}
		f^*_t \in \arg\min_{f_t \in \cF_t}  \int | f_t(s_{0:t}) - \cT[\hat{f}_{t+1}] (s_{0:t}) |^2 \beta(ds_{1:t}|s_0),
	\end{equation}
	where the Bellman operator $\cT[\hat{f}_{t+1}]$ relies on the true conditional measures, instead of the empirical measures.
	
	Assumption \ref{assum:bound_K} imposes that functions are bounded by $K$. Then the loss function satisfies
	$$\big| \cL (f_t(s_{0:t}), \cT[\hat{f}_{t+1}](s_{0:t})) - \cL (f^*_t(s_{0:t}), \cT[\hat{f}_{t+1}](s_{0:t})) \big| \leq 4K^2,$$
	for $f_t \in \cF_t$. 
	
	By Lemma \ref{lem:mohri} based on \citet[Theorem 3.3]{mohri2018foundations}, with probability at least $1-\delta$ over the draw of an i.i.d. sample $\{ s^{n_t}_{0:t} \}$ of size $N$, the following inequality holds for any $f_t \in \cF_t$:
	\begin{align}
		& \int \cL (f_t(s_{0:t}), \cT[\hat{f}_{t+1}](s_{0:t})) \beta(ds_{1:t}|s_0) - \int \cL (f^*_t(s_{0:t}), \cT[\hat{f}_{t+1}](s_{0:t})) \beta(ds_{1:t}|s_0) \nonumber \\
		& \qquad \leq \frac{1}{N} \sum^N_{n_t = 1} \cL (f_t(s^{n_t}_{0:t}), \cT[\hat{f}_{t+1}](s^{n_t}_{0:t})) - \frac{1}{N} \sum^N_{n_t = 1} \cL (f^*_t(s^{n_t}_{0:t}), \cT[\hat{f}_{t+1}](s^{n_t}_{0:t})) \label{ineq:Rad1} \\
		& \qquad \quad + 2 \E \cR_N \Big\{ \cL (f_t(s_{0:t}), \cT[\hat{f}_{t+1}](s_{0:t})) - \cL (f^*_t(s_{0:t}), \cT[\hat{f}_{t+1}](s_{0:t})) \Big| f_t \in \cF_t \Big\} \nonumber \\
		& \qquad \quad + 4 K^2 \sqrt{\frac{2\ln(1/\delta)}{N}}. \nonumber
	\end{align}
	Next, we bound the right-hand side separately. 
	
	By Assumption \ref{assum:Lip} and Lemma \ref{lem:Lip_shadow}, with a given path $s^{n_t}_{0:t}$ and any $f_{t+1} \in \cF_{t+1}$, we have the following bound between empirical and true Bellman operators:
	\begin{align*}
		|\cT[f_{t+1}](s^{n_t}_{0:t}) - \widehat{\cT}_B[f_{t+1}](s^{n_t}_{0:t})| \leq L \Delta(s^{n_t}_{0:t}; B).
	\end{align*} 
	In addition, the squared loss satisfies
	\begin{align}
		& \Big| \cL (f_t, \widehat{\cT}_B[\hat{f}_{t+1}]) - \cL (f_t, \cT[\hat{f}_{t+1}]) \Big| = \Big| \widehat{\cT}^2_B[\hat{f}_{t+1}] - \cT^2[\hat{f}_{t+1}] - 2 f_t \widehat{\cT}_B[\hat{f}_{t+1}]  + 2 f_t \cT[\hat{f}_{t+1}]  \Big| \nonumber \\
		& \qquad \leq \Big| \widehat{\cT}_B[\hat{f}_{t+1}] - \cT[\hat{f}_{t+1}] \Big| \Big(|\widehat{\cT}_B[\hat{f}_{t+1}]| + |\cT[\hat{f}_{t+1}]| + 2 |f_t| \Big). \label{eq:Lip-loss}
	\end{align}
	Under a given path $s^{n_t}_{0:t}$, we can bound the empirical Bellman operator by
	\begin{align*}
		& |\widehat{\cT}_B[\hat{f}_{t+1}]| = |\widehat{\cT}_B[\hat{f}_{t+1}] - \cT[\hat{f}_{t+1}] + \cT[\hat{f}_{t+1}]| \leq |\widehat{\cT}_B[\hat{f}_{t+1}] - \cT[\hat{f}_{t+1}]| + |\cT[\hat{f}_{t+1}]| \\
		& \qquad \leq L \Delta(s^{n_t}_{0:t}; B) + |\cT[\hat{f}_{t+1}]|.
	\end{align*}
	Therefore,
	\begin{align}
		& \Big| \cL (f_t, \widehat{\cT}_B[\hat{f}_{t+1}]) - \cL (f_t, \cT[\hat{f}_{t+1}]) \Big| \leq L \Delta(s^{n_t}_{0:t}; B) \times (4 K + L \Delta(s^{n_t}_{0:t}; B)). \label{ineq:loss_bound}
	\end{align}
	It implies that
	\begin{align*}
		& \left| \frac{1}{N} \sum^N_{n_t = 1} \cL (f_t(s^{n_t}_{0:t}), \widehat{\cT}_B[\hat{f}_{t+1}](s^{n_t}_{0:t})) - \frac{1}{N} \sum^N_{n_t = 1} \cL (f_t(s^{n_t}_{0:t}), \cT[\hat{f}_{t+1}](s^{n_t}_{0:t})) \right| \\
		& \qquad \leq \frac{1}{N} \sum^N_{n_t = 1} L \Delta(s^{n_t}_{0:t}; B) \times (4 K + L \Delta(s^{n_t}_{0:t}; B)).
	\end{align*}
	Hence, the first two terms on the right-hand side of \eqref{ineq:Rad1} satisfy
	\begin{align*}
		& \frac{1}{N} \sum^N_{n_t = 1} \cL (f_t(s^{n_t}_{0:t}), \cT[\hat{f}_{t+1}](s^{n_t}_{0:t})) - \frac{1}{N} \sum^N_{n_t = 1} \cL (f^*_t(s^{n_t}_{0:t}), \cT[\hat{f}_{t+1}](s^{n_t}_{0:t})) \\
		& \qquad \leq \frac{1}{N} \sum^N_{n_t = 1} \cL (f_t(s^{n_t}_{0:t}), \widehat{\cT}_B[\hat{f}_{t+1}](s^{n_t}_{0:t})) - \frac{1}{N} \sum^N_{n_t = 1} \cL (f^*_t(s^{n_t}_{0:t}), \widehat{\cT}_B[\hat{f}_{t+1}](s^{n_t}_{0:t})) \\
		&\qquad \quad + \frac{2}{N} \sum^N_{n_t = 1} L \Delta(s^{n_t}_{0:t}; B) \times (4 K + L \Delta(s^{n_t}_{0:t}; B)).
	\end{align*}
	
	To simplify the Rademacher complexity in \eqref{ineq:Rad1}, we first note that
	\begin{align*}
		& \E \cR_N \Big\{ \cL (f_t(s_{0:t}), \cT[\hat{f}_{t+1}](s_{0:t})) - \cL (f^*_t(s_{0:t}), \cT[\hat{f}_{t+1}](s_{0:t})) \Big| f_t \in \cF_t \Big\} \\
		& \qquad = \E \cR_N \Big\{ \cL (f_t(s_{0:t}), \cT[\hat{f}_{t+1}](s_{0:t})) \Big| f_t \in \cF_t \Big\},
	\end{align*}
	thanks to the symmetry of Rademacher random variables. Moreover, we can prove that the squared loss $\cL (f_t(s_{0:t}), \cT[\hat{f}_{t+1}](s_{0:t}))$ is Lipschitz in the first argument $f_t(s_{0:t})$ with a Lipschitz constant of $4K$. The proof is similar to \eqref{eq:Lip-loss}. Then the contraction property of Rademacher complexity, as described in \citet[Lemma 26.9]{shalev2014understanding} or Lemma \ref{lem:contraction}, yields
	\begin{align*}
		\E \cR_N \Big\{ \cL (f_t(s_{0:t}), \cT[\hat{f}_{t+1}](s_{0:t})) \Big| f_t \in \cF_t \Big\} \leq 4K \E\cR_N \{ f_t | f_t \in \cF_t \} =: 4K \E\cR_N \cF_t.
	\end{align*}
	
	As a result, the inequality \eqref{ineq:Rad1} becomes
	\begin{align*}
		& \int \cL (f_t(s_{0:t}), \cT[\hat{f}_{t+1}](s_{0:t})) \beta(ds_{1:t}|s_0) - \int \cL (f^*_t(s_{0:t}), \cT[\hat{f}_{t+1}](s_{0:t})) \beta(ds_{1:t}|s_0) \\
		& \qquad \leq \frac{1}{N} \sum^N_{n_t = 1} \cL (f_t(s^{n_t}_{0:t}), \widehat{\cT}_B[\hat{f}_{t+1}](s^{n_t}_{0:t})) - \frac{1}{N} \sum^N_{n_t = 1} \cL (f^*_t(s^{n_t}_{0:t}), \widehat{\cT}_B[\hat{f}_{t+1}](s^{n_t}_{0:t})) \\
		&\qquad \quad + \frac{2}{N} \sum^N_{n_t = 1} L \Delta(s^{n_t}_{0:t}; B) \times (4 K + L \Delta(s^{n_t}_{0:t}; B)) + 8K \E\cR_N \cF_t + 4 K^2 \sqrt{\frac{2\ln(1/\delta)}{N}}.
	\end{align*}
	We take $f_t = \hat{f}_t$. By definition, $\hat{f}_t$ is an optimizer for the empirical loss \eqref{eq:emp_loss}. Then
	\begin{align*}
		\frac{1}{N} \sum^N_{n_t = 1} \cL (\hat{f}_t(s^{n_t}_{0:t}), \widehat{\cT}_B[\hat{f}_{t+1}](s^{n_t}_{0:t})) - \frac{1}{N} \sum^N_{n_t = 1} \cL (f^*_t(s^{n_t}_{0:t}), \widehat{\cT}_B[\hat{f}_{t+1}](s^{n_t}_{0:t}))  \leq 0.
	\end{align*}
	Besides, with Assumption \ref{assum:complete}, we have
	\begin{align*}
		\int \cL (f^*_t(s_{0:t}), \cT[\hat{f}_{t+1}](s_{0:t})) \beta(ds_{1:t}|s_0) \leq \zeta.
	\end{align*}
	Hence, the inequality \eqref{ineq:Rad1} becomes, with probability at least $1 - \delta$,
	\begin{align}
		& \int \cL (\hat{f}_t(s_{0:t}), \cT[\hat{f}_{t+1}](s_{0:t})) \beta(ds_{1:t}|s_0) \label{ineq:temp1}\\
		& \qquad \leq \zeta + \frac{2}{N} \sum^N_{n_t = 1} L \Delta(s^{n_t}_{0:t}; B) \times (4 K + L \Delta(s^{n_t}_{0:t}; B)) + 8K \E\cR_N \cF_t + 4 K^2 \sqrt{\frac{2\ln(1/\delta)}{N}}. \nonumber
	\end{align}
	
	For the Bellman error $\cB(\hat{f})$, we apply the union bound. With probability at least $1 - \delta$, 
	\begin{align*}
		\cB(\hat{f}) \leq & T \zeta + \frac{2 L^2}{N} \sum^{T-1}_{t=0} \sum^N_{n_t = 1} \Delta^2(s^{n_t}_{0:t}; B) + \frac{8 L K}{N} \sum^{T-1}_{t=0} \sum^N_{n_t = 1} \Delta(s^{n_t}_{0:t}; B) \\
		& + 8 K \sum^{T-1}_{t=0} \E \cR_N \cF_t  + 4 T K^2 \sqrt{\frac{2 \ln(T/\delta)}{N}}.
	\end{align*}
\end{proof}

\begin{proof}[Proof of Theorem \ref{thm:light_RC}]
	By Assumption \ref{assum:light} on the dynamics of $\mu$ and $\nu$, we have
	\begin{align*}
		& W^p_p(\mu(dx_{t+1}| x^{n_t}_{0:t}), \hat{\mu}_B(dx_{t+1}| x^{n_t}_{0:t})) = W^p_p(\lambda(dx_{t+1}| x^{n_t}_{0:t}), \hat{\lambda}_B(dx_{t+1}| x^{n_t}_{0:t})), \\
		& W^p_p(\nu(dy_{t+1}| y^{n_t}_{0:t}), \hat{\nu}_B(dy_{t+1}| y^{n_t}_{0:t})) = W^p_p(\eta(dy_{t+1}| y^{n_t}_{0:t}), \hat{\eta}_B(dy_{t+1}| y^{n_t}_{0:t})).
	\end{align*}
	
	Conditioned on $s^{n_t}_{0:t}$, with the uniform moment bound on $\eta$ and $\lambda$ in Assumption \ref{assum:light}, \citet[Theorem 2]{fournier2015rate} guarantees that with probability at least $1 - \delta_1$, 
	\begin{align*}
		& W^p_p(\lambda(dx_{t+1}| x^{n_t}_{0:t}), \hat{\lambda}_B(dx_{t+1}| x^{n_t}_{0:t}))  \leq R(\delta_1, B).
	\end{align*}
	Similarly, with probability at least $1 - \delta_2$,
	\begin{align*}
		W^p_p(\eta(dy_{t+1}| y^{n_t}_{0:t}), \hat{\eta}_B(dy_{t+1}| y^{n_t}_{0:t})) \leq R(\delta_2, B).
	\end{align*}
	For a generic time $t$, the inequality \eqref{ineq:temp1} holds with probability at least $1 - \delta$. In addition, \eqref{ineq:temp1} relies on $N$ samples of $s^{n_t}_{0:t}$. Hence, we apply the union bound by replacing $\delta$ with $\delta/(2N+1)$ and obtain  
	\begin{align*}
		& \int \cL (\hat{f}_t(s_{0:t}), \cT[\hat{f}_{t+1}](s_{0:t})) \beta(ds_{1:t}|s_0) \\
		& \qquad \leq \zeta + 2 L^2 \cdot 2^{2/p} R\left(\frac{\delta}{2N+1}, B \right)^{2/p} + 8LK \cdot 2^{1/p} R\left(\frac{\delta}{2N+1}, B \right)^{1/p} \\
		&\qquad \quad + 8K \E\cR_N \cF_t + 4 K^2 \sqrt{\frac{2\ln((2N+1)/\delta)}{N}},
	\end{align*}
	with probability at least $1 - \delta$.
	
	The bound on $\cB(\hat{f})$ follows by applying the union bound again and replacing $\delta$ with $\delta/T$. 
\end{proof}

\begin{proof}[Proof of Corollary \ref{cor:entropy_RC}]
	We denote $C$ as a constant that may vary line by line.
	
	With compact domains and the Lipschitz property, for any $f_{t+1}$, \citet[Theorem 1]{genevay2019sample} proved that
	\begin{align*}
		0 \leq \cT_\varepsilon[f_{t+1}](s^{n_t}_{0:t}) - \cT [f_{t+1}](s^{n_t}_{0:t}) \leq C \varepsilon \ln (C/\varepsilon),
	\end{align*}
	where $C$ is independent of the states $s^{n_t}_{0:t}$.
	
	The $\cC^\infty$ smooth property of $f_t$ and the boundedness of derivatives are used in \citet[Theorems 2 and 3]{genevay2019sample}. Building upon these results, \citet[Corollary 1]{genevay2019sample} proved that, with probability at least $1 - \delta$, 
	\begin{align*}
		& | \cT_\varepsilon[f_{t+1}](s^{n_t}_{0:t}) - \widehat{\cT}_{\varepsilon, B}[f_{t+1}](s^{n_t}_{0:t}) | \\
		& \qquad \leq 6 \left( 1 + e^{C/\varepsilon} \right) \max\{1, \varepsilon^{-d/2}\}  \times C/\sqrt{B} + C(1 + \varepsilon e^{C/\varepsilon}) \sqrt{\frac{\ln(1/\delta)}{B}} \\
		& \qquad \leq \frac{C}{\sqrt{B}} \left( 1 + \max\{1, \varepsilon \} e^{C/\varepsilon}  \right) \left( \max\{1, \varepsilon^{-d/2}\} + \sqrt{\ln\left(1/\delta\right)} \right).
	\end{align*}
	Therefore, with probability at least $1 - \delta$, we have
	\begin{align*}
		| \cT[f_{t+1}](s^{n_t}_{0:t}) - \widehat{\cT}_{\varepsilon, B}[f_{t+1}](s^{n_t}_{0:t}) | \leq R(\delta, B; \varepsilon).
	\end{align*}
	Similar to \eqref{ineq:loss_bound}, we can show 
	\begin{align}
		& \Big| \cL (f_t, \widehat{\cT}_{\varepsilon, B}[\hat{f}_{t+1}]) - \cL (f_t, \cT[\hat{f}_{t+1}]) \Big| \leq R(\delta, B; \varepsilon) \times (4 K + R(\delta, B; \varepsilon)).
	\end{align}
	The remaining proof is almost the same as in Lemma \ref{lem:RC} and Theorem \ref{thm:light_RC}.
\end{proof}

\begin{proof}[Proof of Lemma \ref{lem:LRC}]
	The main idea is to apply \citet[Theorem 3.3]{bartlett2005local} with a properly chosen function class. At time $t$ when $\hat{f}_{t+1}$ is obtained, we consider the function class as
	\begin{equation*}
		\widetilde{\cF}_t := \Big\{ \cL (f_t(s_{0:t}), \cT[\hat{f}_{t+1}](s_{0:t})) - \cL (f^*_t(s_{0:t}), \cT[\hat{f}_{t+1}](s_{0:t})) \Big| f_t \in \cF_t \Big\}.
	\end{equation*}
	We verify the assumptions in \citet[Theorem 3.3]{bartlett2005local}. First, we need to find a functional $\cM: \widetilde{\cF}_t \rightarrow \R^+$ and a constant $R_c$ such that for every $g \in \widetilde{\cF}_t$, $\text{Var}[g] \leq \cM(g) \leq R_c \E[g]$. The expectation and the variance are taken under $\beta(ds_{1:t}|s_0)$. Note that
	\begin{align*}
		& \text{Var} \left[ \cL (f_t(s_{0:t}), \cT[\hat{f}_{t+1}](s_{0:t})) - \cL(f^*_t(s_{0:t}), \cT[\hat{f}_{t+1}](s_{0:t}))\right] \\
		& \qquad \leq \E \left[ \left( \cL (f_t(s_{0:t}), \cT[\hat{f}_{t+1}](s_{0:t})) - \cL(f^*_t(s_{0:t}), \cT[\hat{f}_{t+1}](s_{0:t})) \right)^2\right] \\
		& \qquad \leq 16 K^2 \E \left[ \left(f_t(s_{0:t}) - f^*_t(s_{0:t}) \right)^2\right].
	\end{align*}
	The second inequality follows from the fact that $\cL$ is Lipschitz on the first argument with constant $4K$, where Assumption \ref{assum:bound_K} is used. Moreover, the quadratic loss has the following property:
	\begin{align*}
		& \frac{\E \left[\cL (f_t(s_{0:t}), \cT[\hat{f}_{t+1}](s_{0:t})) \right] + \E\left[\cL(f^*_t(s_{0:t}), \cT[\hat{f}_{t+1}](s_{0:t}))\right]}{2} \\
		& \qquad = \E \left[ \left(   \frac{f_t(s_{0:t}) + f^*_t(s_{0:t})}{2}   -   \cT[\hat{f}_{t+1}](s_{0:t}) \right)^2\right] + \frac{ \E \left[ \left(f_t(s_{0:t}) - f^*_t(s_{0:t}) \right)^2\right]}{4} \\
		& \qquad \geq \E \left[ \left(  f^*_t(s_{0:t})   -   \cT[\hat{f}_{t+1}](s_{0:t}) \right)^2\right] + \frac{ \E \left[ \left(f_t(s_{0:t}) - f^*_t(s_{0:t}) \right)^2\right]}{4}.
	\end{align*}
	The second inequality holds because $f^*_t$ is an optimizer. It implies that
	\begin{align*}
		& \E \left[\cL (f_t(s_{0:t}), \cT[\hat{f}_{t+1}](s_{0:t})) \right] - \E\left[\cL(f^*_t(s_{0:t}), \cT[\hat{f}_{t+1}](s_{0:t}))\right] \\
		& \qquad \geq \frac{ \E \left[ \left(f_t(s_{0:t}) - f^*_t(s_{0:t}) \right)^2\right]}{2}.
	\end{align*}
	Therefore,
	\begin{align*}
		& \text{Var} \left[ \cL (f_t(s_{0:t}), \cT[\hat{f}_{t+1}](s_{0:t})) - \cL(f^*_t(s_{0:t}), \cT[\hat{f}_{t+1}](s_{0:t}))\right] \\
		&\qquad \leq 32 K^2 \E \left[\cL (f_t(s_{0:t}), \cT[\hat{f}_{t+1}](s_{0:t})) - \cL(f^*_t(s_{0:t}), \cT[\hat{f}_{t+1}](s_{0:t}))\right].
	\end{align*}
	Then we can choose the constant $R_c = 32 K^2$ and the functional $\cM$ as
	\begin{equation*}
		\cM(f) := 16 K^2 \E \left[ \left(f_t(s_{0:t}) - f^*_t(s_{0:t}) \right)^2\right].
	\end{equation*}
	By \citet[Theorem 3.3]{bartlett2005local}, we assume there exists a sub-root function $\varphi_t$ with the fixed point $r^*_\varphi$. For any $r \geq r^*_\varphi$, $\varphi_t$ satisfies
	\begin{equation*}
		\varphi_t(r) \geq 32 K^2 \E\cR_N \Big\{ \cL (f_t(s_{0:t}), \cT[\hat{f}_{t+1}](s_{0:t})) - \cL (f^*_t(s_{0:t}), \cT[\hat{f}_{t+1}](s_{0:t})) \Big| f_t \in \cF_t, \; \cM(f_t) \leq r  \Big\}.
	\end{equation*}
	We can simplify the right-hand side with the contraction property of Rademacher complexity:
	\begin{align*}
		& \E\cR_N \Big\{ \cL (f_t(s_{0:t}), \cT[\hat{f}_{t+1}](s_{0:t})) - \cL (f^*_t(s_{0:t}), \cT[\hat{f}_{t+1}](s_{0:t})) \Big| f_t \in \cF_t, \; \cM(f_t) \leq r  \Big\} \\
		& \qquad \leq 4K \E\cR_N \Big\{ f_t(s_{0:t}) - f^*_t(s_{0:t}) \Big| f_t \in \cF_t, \; \int \left(f_t(s_{0:t}) - f^*_t(s_{0:t}) \right)^2 \beta(ds_{1:t}| s_0) \leq \frac{r}{16K^2}  \Big\}.
	\end{align*}
	Hence, we can adopt a simpler sub-root function $\psi_t(r)$ such that
	\begin{align*}
		\psi_t(r) \geq  \E\cR_N \Big\{ f_t(s_{0:t}) - f^*_t(s_{0:t}) \Big| f_t \in \cF_t, \; \int \left(f_t(s_{0:t}) - f^*_t(s_{0:t}) \right)^2 \beta(ds_{1:t}| s_0) \leq r  \Big\}.
	\end{align*}
	It is the sub-root function defined in Assumption \ref{assum:sub-root}. Moreover, we can set $\varphi_t(r) = 128 K^3 \psi_t(\frac{r}{16K^2})$. If we denote the fixed point of $\psi_t(r)$ as $r^*_t$, then $r^*_\varphi \leq 1024 K^4 r^*_t$ by Lemma \ref{lem:subroot} from \citet[Lemma G.5]{duan2021risk} using the definition of sub-root functions.
	
	Now we can apply \citet[Theorem 3.3]{bartlett2005local} to the function class $\widetilde{\cF}_t$ with the sub-root function $\varphi_t(r)$. Denote $c_1 = 704$ and $c_2 = 26$. For any $\theta > 1$ and $\delta \in (0, 1)$, with probability at least $1-\delta$ over the draw of an i.i.d. sample $\{ s^{n_t}_{0:t} \}$ of size $N$, the following inequality holds for any $f_t \in \cF_t$:
	\begin{align}
		& \int \cL (f_t(s_{0:t}), \cT[\hat{f}_{t+1}](s_{0:t})) \beta(ds_{1:t}|s_0) - \int \cL (f^*_t(s_{0:t}), \cT[\hat{f}_{t+1}](s_{0:t})) \beta(ds_{1:t}|s_0) \nonumber \\
		& \qquad \leq \frac{\theta}{\theta - 1} \left( \frac{1}{N} \sum^N_{n_t = 1} \cL (f_t(s^{n_t}_{0:t}), \cT[\hat{f}_{t+1}](s^{n_t}_{0:t})) - \frac{1}{N} \sum^N_{n_t = 1} \cL (f^*_t(s^{n_t}_{0:t}), \cT[\hat{f}_{t+1}](s^{n_t}_{0:t})) \right) \label{ineq:LRC1} \\
		&\qquad  \quad + \frac{c_1 \theta}{R_c} r^*_\varphi + \frac{11 (b - a) + R_c  \theta c_2}{N} \ln(1/\delta), \nonumber
	\end{align}
	where $[a, b]$ is the range of functions in $\widetilde{\cF}_t$. We can simplify the right-hand side using the following facts:
	\begin{enumerate}[label={(\arabic*)}]
		\item Similar to Lemma \ref{lem:RC}, we can use the inequality \eqref{ineq:loss_bound} and the fact that $\hat{f}_t$ is an optimizer of the empirical loss to simplify the first two terms on the right-hand side;
		\item We recall the constants used are given by $c_1 = 704$, $c_2=26$, $R_c = 32 K^2$, $a = - 4K^2$, $b = 4K^2$, $r^*_\varphi \leq 1024 K^4 r^*_t$.
	\end{enumerate}
	We obtain
	\begin{align}
		& \int \cL (\hat{f}_t(s_{0:t}), \cT[\hat{f}_{t+1}](s_{0:t})) \beta(ds_{1:t}|s_0) \nonumber \\
		& \qquad \leq \zeta + \frac{\theta}{\theta - 1} \left( 2 \frac{L^2}{N} \sum^N_{n_t = 1} \Delta^2(s^{n_t}_{0:t}; B) + 8 \frac{L K}{N}  \sum^N_{n_t = 1} \Delta(s^{n_t}_{0:t}; B) \right) \label{ineq:LRC2} \\
		& \qquad \quad + 22528 \theta K^2 r^*_t + \frac{(88 + 832 \theta ) K^2}{N} \ln(1/\delta). \nonumber
	\end{align}
	Then the result follows with the union bound. 
\end{proof}

\section{Proofs of results in Section \ref{sec:NN}}

\begin{proof}[Proof of Proposition \ref{prop:radius_LRC}]
	We first fix $N$ sample paths as $s^1_{0:t}, ..., s^N_{0:t}$ and consider the corresponding empirical measure $\beta_N$. Denote $\cG := \cF_t - \cF_t = \{f_1 - f_2: f_1, f_2 \in \cF_t \}$. We omit the time subscript in the following proof since it is clear.
	
	We pick a small $\varepsilon > 0$ satisfying Proposition \ref{prop:cover_ineq}. Then the empirical LRC satisfies
	\begin{align}
		& \E_\sigma \cR_N \left\{ \tilde{g}: \tilde{g} \in \cG - \cG, \; \|\tilde{g} \|_{L_2(\beta_N)} \leq \varepsilon \right\} \nonumber \\
		& \qquad \leq  C \sum^\infty_{k=1} \frac{\varepsilon}{2^{k-1}} \sqrt{\frac{\ln \cN(\varepsilon/2^k, \cG - \cG, \| \cdot \|_{L_2(\beta_N)} )}{N}} \label{ineq1}\\
		& \qquad \leq C \sum^\infty_{k=1} \frac{\varepsilon}{2^{k-1}} \sqrt{\frac{2 \ln \cN(\varepsilon/2^{k+1}, \cG, \| \cdot \|_{L_2(\beta_N)} )}{N}} \label{ineq2}\\
		&\qquad \leq C \sum^\infty_{k=1} \frac{\varepsilon}{2^{k-1}} \sqrt{\frac{4 \ln \cN(\varepsilon/2^{k+2}, \cF_t, \| \cdot \|_{L_2(\beta_N)} )}{N}} \label{ineq3} \\
		& \qquad \leq C \sum^\infty_{k=1} \frac{\varepsilon}{2^{k-1}} \sqrt{\frac{4 \ln \cN_\infty(\varepsilon/2^{k+2}, \cF_t)}{N}} \label{ineq4} \\
		&\qquad \leq \frac{C}{\sqrt{N}} \sum^\infty_{k=1} \frac{\varepsilon}{2^{k-1}} \sqrt{C_{NN} \left( 1 + \ln\left( \frac{2^{k+2}}{\varepsilon} \right) \right)} \label{ineq5} \\
		&\qquad \lesssim \frac{\varepsilon}{\sqrt{N}} (1 + \sqrt{\ln(1/\varepsilon)}), \label{ineq6}
	\end{align}
	with an absolute constant $C>0$. The first inequality \eqref{ineq1} follows from the proof of Dudley entropy integral, see \citet[Theorem 5.22]{wainwright2019high} or \citet[Lemma A.5]{lei2016local}. The second and the third inequalities \eqref{ineq2} -- \eqref{ineq3} use Lemma \ref{lem:diff_cover}. The fourth inequality \eqref{ineq4} is due to Lemma \ref{lem:inf_cover}. The fifth inequality \eqref{ineq5} is from Proposition \ref{prop:cover_ineq}. The last inequality \eqref{ineq6} is due to $\sum^\infty_{k=1} \frac{\sqrt{k}}{2^{k-1}} < \infty$ and $\sqrt{a+b} \leq \sqrt{a} + \sqrt{b}$, $a, b \geq 0$. Since the right-hand side does not depend on the sample $s^1_{0:t}, ..., s^N_{0:t}$, we have
	\begin{align*}
		& \E \cR_N \left\{ \tilde{g}: \tilde{g} \in \cG - \cG, \; \|\tilde{g}\|_{L_2(\beta_N)} \leq \varepsilon \right\} = \E\E_\sigma \cR_N \left\{ \tilde{g}: \tilde{g} \in \cG - \cG, \; \|\tilde{g}\|_{L_2(\beta_N)} \leq \varepsilon \right\} \\
		&\qquad \lesssim \frac{\varepsilon}{\sqrt{N}} (1 + \sqrt{\ln(1/\varepsilon)}).
	\end{align*}
	
	By \citet[Theorem 1]{lei2016local}, we have
	\begin{align*}
		& \E \cR_N \left\{ g: g \in \cG, \; \int g^2 d\beta \leq r \right\} \\
		& \qquad \leq 2 \E \cR_N \left\{ \tilde{g}: \tilde{g} \in \cG - \cG, \; \|\tilde{g}\|_{L_2(\beta_N)} \leq \varepsilon \right\} + \frac{8 K_t \ln \cN_2(\varepsilon/2, \cG)}{N} + \sqrt{\frac{2r \ln \cN_2(\varepsilon/2, \cG)}{N}} \\
		& \qquad \lesssim \frac{\varepsilon}{\sqrt{N}} (1 + \sqrt{\ln(1/\varepsilon)}) +  \frac{1 + \ln(1/\varepsilon)}{N} + \sqrt{ \frac{r(1 + \ln(1/\varepsilon) )}{N}} \\
		& \qquad  \lesssim \frac{\ln N}{N} + \sqrt{ \frac{\ln N}{N} r},
	\end{align*}	
	where we set $\varepsilon \propto 1/\sqrt{N}$ in the last inequality. 
	Therefore, we can choose $\psi_t(r)  = C (\ln(N)/N + \sqrt{r \ln(N)/N} )$ with a large enough constant $C$ independent of $r$. $\psi_t(r)$ is a sub-root function. Solving $\psi_t(r) = r$ shows that the fixed point satisfies $r^*_t \lesssim \frac{\ln(N)}{N}$.
	
\end{proof}

\begin{proof}[Proof of Proposition \ref{prop:Holder}]
	The first step is in the same spirit as Lemma \ref{lem:Lip_shadow}. Fix two states $s_{0:t}, s'_{0:t} \in \cS_{0:t}$. Denote an optimizer of $\cT[f_{t+1}](s_{0:t})$ as $\pi^*(ds_{t+1}|s_{0:t})$. Consider a coupling with disintegration $\mu(dx_{t+1}| x_{0:t}) \otimes \kappa^\mu $ attaining $W_p(\mu(dx_{t+1}| x_{0:t}), \mu(dx_{t+1}| x'_{0:t}))$ and another coupling with disintegration $\nu(dy_{t+1}| y_{0:t}) \otimes \kappa^\nu$ attaining $W_p(\nu(dy_{t+1}| y_{0:t}), \nu(dy_{t+1}| y'_{0:t}))$. Consider $S(\pi^*)$ as the $(\kappa^\mu, \kappa^\nu)$-shadow of $\pi^*(ds_{t+1}|s_{0:t})$. Then $S(\pi^*) \in \Pi(\mu^t, \nu^t, s'_{0, t})$. We obtain
	\begin{align*}
		& \cT[f_{t+1}](s'_{0:t}) - \cT[f_{t+1}](s_{0:t}) \\
		& \qquad \leq \int f_{t+1}(s'_{0:t}, s_{t+1}) S(\pi^*)(ds_{t+1}|s'_{0:t}) - \int f_{t+1} (s_{0:t}, s_{t+1}) \pi^*(ds_{t+1}|s_{0:t}) \\
		& \qquad \leq L_{t+1} W_p (\pi^*(ds_{t+1}| s_{0:t}), S(\pi^*)(ds_{t+1}| s'_{0:t})) +  L_{t+1} \|s_{0:t} - s'_{0:t} \|^\alpha_\infty \\
		& \qquad \leq L_{t+1} \Delta(s_{0:t}, s'_{0:t})  +  L_{t+1} \|s_{0:t} - s'_{0:t} \|^\alpha_\infty.
	\end{align*}  
	The second inequality holds due to Assumption \ref{assum:cond_Lip}. The last inequality is from Lemma \ref{lem:shadow}. The other side can be proved similarly. Then
	\begin{align*}
		| \cT[f_{t+1}](s_{0:t}) - \cT[f_{t+1}](s'_{0:t}) | \leq L_{t+1} \Delta(s_{0:t}, s'_{0:t})  +  L_{t+1} \|s_{0:t} - s'_{0:t} \|^\alpha_\infty.
	\end{align*}
	The inequality \eqref{ineq:Holder} yields
	\begin{align*}
		& \sup_{s_{0:t}, s'_{0:t} \in \cS_{0:t}, \; s_{0:t} \neq s'_{0:t}}  \frac{	| \cT[f_{t+1}](s_{0:t}) - \cT[f_{t+1}](s'_{0:t}) |}{\|s_{0:t} - s'_{0:t}\|^\alpha_\infty} \\
		& \qquad \leq L_{t+1} \times (H_t - K_{t+1} - L_{t+1})/L_{t+1} + L_{t+1} \\
		& \qquad = H_t - K_{t+1}.
	\end{align*}
	As $\| \cT[f_{t+1}](s_{0:t}) \|_\infty \leq K_{t+1}$ with the boundedness assumption, it follows that $\cT[f_{t+1}] \in \cH^\alpha (\cS_{0:t}, H_t)$ by the Definition \ref{def:Holder}.
\end{proof}

\begin{proof}[Proof of Proposition \ref{prop:ReLU_complete}]
	We prove the claims by backward induction. 
	
	At time $T-1$, since $\cF_T$ only contains the cost function $c(s_{1:T})$ and conditions (2) and (4) hold, then Proposition \ref{prop:Holder} shows that $\cT[f_T] \in \cH^\alpha (\cS_{0:T-1}, H_{T-1})$ with a sufficiently large $H_{T-1} > K_T + L_T$. Hence, we can apply  \citet[Theorem 5]{schmidt2020}: There exists a sparse ReLU network $\tilde{f} \in \cF(D_{T-1}, \{ q^{T-1}_j\}^{D_{T-1}+1}_{j=0}, \gamma_{T-1}, \infty)$ with parameters defined in (b), such that
	\begin{align*}
		\| \tilde f - \cT[f_T] \|_\infty & \leq  (2H_{T-1} + 1)(1+ d^2_{T-1} + \alpha^2) 6^{d_{T-1}} G_{T-1} 2^{-m_{T-1}} + H_{T-1} 3^\alpha G_{T-1}^{-\frac{\alpha}{d_{T-1}}} \\
		& := K_{T-1} - K_{T}.
	\end{align*}
	Since $\| \cT[f_T] \|_\infty \leq K_T$, then we can restrict the function class $\cF_{T-1}$ to be bounded by $K_{T-1}$. Therefore,
	\begin{align*}
		& \sup_{f_T \in \cF_T} \inf_{f_{T-1} \in \cF_{T-1}} \int \big( f_{T-1}(s_{0:T-1}) - \cT[f_T](s_{0:T-1}) \big)^2 \beta(ds_{1:T-1}| s_0) \\
		& \qquad \leq \sup_{f_{T} \in \cF_{T}} \inf_{f_{T-1} \in \cF_{T-1}} \| f_{T-1}(s_{0:T-1}) - \cT[f_T](s_{0:T-1}) \|^2_\infty \\
		& \qquad \leq  (K_{T-1} - K_T)^2 := \zeta_{T-1}.
	\end{align*}
	Then claim (c) holds with $\zeta_{T-1}$. 
	
	At time $T-2$, we first observe that a ReLU network is a composition of Lipschitz functions and thus is also Lipschitz. Furthermore, under condition (1), the domain is compact. Then Assumption \ref{assum:cond_Lip} holds for functions in $\cF_{T-1}$ with a sufficiently large $L_{T-1}$, since the inputs and parameters of sparse ReLU networks are uniformly bounded. In Condition (4), we can choose $H_{T-2} > L_{T-1} + K_{T-1}$. It is important to note that $L_{T-1}$ and $K_{T-1}$ are determined by the parameters from the previous step $T$. Proposition \ref{prop:Holder} shows that $\cT[f_{T-1}] \in \cH^\alpha (\cS_{0:T-2}, H_{T-2})$ for $f_{T-1} \in \cF_{T-1}$. The remaining proof repeats the previous arguments. Hence, we obtain the results for $t=T-2, ..., 0$ in a backward manner.
\end{proof}

\begin{proof}[Proof of Proposition \ref{prop:sigmoid}]
	We prove the results backwardly in time $t$. 
	
	At time $T-1$, we first note that $\cF_T$ only contains the cost function $c(s_{1:T})$. Under conditions (2) and (4), we have $\cT[f_{T}] \in \cH^\alpha(\cS_{0:T-1}, H_{T-1})$ by Proposition \ref{prop:Holder}, where the constant $H_{T-1} > L_T + K_T$. By Lemma \ref{lem:extension}, we can extend $\cT[f_{T}]$ to the whole domain $\R^{Td}$ as follows:
	\begin{align*}
		g(s_{0:T-1}) := \inf \{ \cT[f_{T}](u_{0:T-1}) + (H_{T-1} - K_T)\| s_{0:T-1} - u_{0:T-1} \|^\alpha_\infty,  \quad u_{0:T-1} \in \cS_{0:T-1} \}.
	\end{align*}
	Since $\cT[f_{T}]$ is $\alpha$-H\"older continuous with constant $H_{T-1} - K_T$ on $\cS_{0:T-1}$, we can verify that $g$ is $\alpha$-H\"older continuous with the same constant $H_{T-1} - K_T$ on $\R^{Td}$. Moreover, $g = \cT[f_{T}]$ on $\cS_{0:T-1}$.
	
	Then we can apply \citet[Theorem 1]{langer2021approximating}: There exists a neural network $f_{T-1} \in \Sigma(D_{T-1}, q_{T-1},$ $ U_{T-1}, \infty)$, with constants defined in (b), such that
	\begin{align*}
		\| f_{T-1} - \cT[f_{T}] \|_{\infty, \, \cS_{0:T-1}} \leq \frac{c_{5, T-1} (\max\{a, K_T\})^3}{M^{2\alpha}_{T-1}},
	\end{align*} 
	where the uniform norm is over $\cS_{0:T-1}$. Then we can choose the uniform upper bound $K_{T-1} = K_T + \frac{c_{5, T-1} (\max\{a, K_T\})^3}{M^{2\alpha}_{T-1}}$. The approximate completeness assumption at time $T-1$ is valid with $(K_{T-1} - K_T)^2$.
	
	For a generic time $t$, note that the sigmoid activation function is Lipschitz. Parameters and inputs of sigmoid networks in $\cF_{t+1}$ are uniformly bounded. Then all functions in $\cF_{t+1}$ are Lipschitz with a universal constant $L_{t+1}$. The constant $H_t > L_{t+1} + K_{t+1}$ is chosen backwardly. Therefore, the remaining proof becomes similar to the arguments above.
\end{proof}

\section{Auxiliary results}
\subsection{Rademacher complexity}
\begin{lemma}(A slight extension to \citet[Theorem 3.3]{mohri2018foundations}) \label{lem:mohri}
	Let $\cG$ be a family of functions mapping from $\cZ$ to $[-a, a]$ for some $a>0$. Then, for any $\delta \in (0, 1)$, with probability at least $1 - \delta$ over the draw of an i.i.d. sample $\{ z_i \}$ of size $N$, the following holds for all $g \in \cG$:
	\begin{align}
		\E [g(z)] \leq \frac{1}{N} \sum^N_{i=1} g(z_i) + 2 \E \cR_N \cG + a \sqrt{\frac{2 \ln(1/\delta)}{N} }.
	\end{align} 
\end{lemma}

\begin{lemma}{\cite[Lemma 26.9, contraction property]{shalev2014understanding}}\label{lem:contraction}
	Let $A \subseteq \R^N$ be a set of vectors. Similar to the empirical Rademacher complexity \eqref{ERC}, define 
	\begin{equation*}
		\E_\sigma \cR_N A := \E_\sigma \left[ \sup_{a \in A}  \frac{1}{N} \sum^N_{i=1} \sigma_i a_i \right]. 
	\end{equation*}
	For each $i = 1, ..., N$, let $\phi_i: \R \rightarrow \R$ be an $L$-Lipschitz function, namely $|\phi_i(x) - \phi_i(y)| \leq L|x - y|$, $\forall x, y \in \R$. For $a \in \R^N$, let $\phi(a)$ denote the vector $(\phi_1(a_1), ..., \phi_N(a_N))$. Let $\phi \circ A := \{\phi(a): a \in A\}$. Then
	\begin{equation}
		\E_\sigma \cR_N(\phi \circ A) \leq L \E_\sigma \cR_N A.
	\end{equation}
\end{lemma}

\subsection{Shadow technique}
We introduce the concept of shadow following \cite{eckstein2022MA}. Let $\mu, \tilde{\mu} \in \cP_p (\cX)$ and $\nu, \tilde{\nu} \in \cP_p(\cY)$.  Let $\Lambda \in \Pi(\mu, \tilde{\mu})$ be a coupling attaining $W_p(\mu, \tilde{\mu})$ and $\Lambda = \mu \otimes \kappa^\mu$ be a disintegration. Similarly, let $\Lambda' \in \Pi(\nu, \tilde{\nu})$ be a coupling attaining $W_p(\nu, \tilde{\nu})$ and $\Lambda' = \nu \otimes \kappa^\nu$ be a disintegration. Denote $\kappa(x, y) := \kappa^\mu(x) \otimes \kappa^\nu(y)$. Given a coupling $\pi \in \Pi(\mu, \nu)$, its {\it shadow} $S(\pi)$ is the second $\cX \times \cY$ marginal of $\pi \otimes \kappa \in \cP(\cX \times \cY \times \cX \times \cY)$. By definition, we have $S(\pi) \in \Pi(\tilde{\mu}, \tilde{\nu})$. In particular, $S(\pi)$ is also called the $(\kappa^\mu, \kappa^\nu)$-shadow of $\pi$ \citep{eckstein2022comp}. 

\begin{lemma}{\cite[Lemma 3.2]{eckstein2022MA}}\label{lem:shadow}
	The coupling $\pi$ and its shadow $S(\pi)$ satisfy 
	\begin{equation}
		W_p(\pi, S(\pi)) = [W^p_p (\mu, \tilde{\mu}) + W^p_p (\nu, \tilde{\nu})]^{1/p}.
	\end{equation}
\end{lemma}
\begin{proof}
	First,
	\begin{align*}
		W^p_p (\pi, S(\pi)) & = \inf_{\gamma \in \Pi(\pi, S(\pi))} \int \big[ d^p_\cX (x, \tilde{x}) + d^p_\cY (y, \tilde{y})  \big] d\gamma \\
		& \leq \int \big[ d^p_\cX (x, \tilde{x}) + d^p_\cY (y, \tilde{y})  \big] \pi(dx, dy) \kappa(x, y, d\tilde{x}, d\tilde{y}) \\
		& = \int d^p_\cX (x, \tilde{x}) \mu(dx) \kappa^\mu(x, d\tilde{x}) + \int d^p_\cY (y, \tilde{y}) \nu(dy) \kappa^\nu(y, d\tilde{y}) \\
		& = W^p_p (\mu, \tilde{\mu}) + W^p_p (\nu, \tilde{\nu}).
	\end{align*}
	The last equality holds since $\mu \otimes \kappa^\mu$ and $\nu \otimes \kappa^\nu$ are optimal couplings by the definition.
	
	For another direction, any $\gamma \in \Pi(\pi, S(\pi))$ induces two couplings $\gamma_1 \in \Pi(\mu, \tilde{\mu})$ and $\gamma_2 \in \Pi(\nu, \tilde{\nu})$. Hence,
	\begin{align*}
		W^p_p (\pi, S(\pi)) & = \inf_{\gamma \in \Pi(\pi, S(\pi))} \int \big[ d^p_\cX (x, \tilde{x}) + d^p_\cY (y, \tilde{y})  \big] d\gamma \\
		& \geq \inf_{\gamma_1 \in \Pi(\mu, \tilde{\mu})} \int d^p_\cX (x, \tilde{x}) d \gamma_1 + \inf_{\gamma_2 \in \Pi(\nu, \tilde{\nu})} \int d^p_\cY (y, \tilde{y}) d\gamma_2 \\
		& = W^p_p (\mu, \tilde{\mu}) + W^p_p (\nu, \tilde{\nu}).
	\end{align*}
	It proves the claim.
\end{proof}

\begin{lemma}\label{lem:Lip_shadow}
	Suppose a function $f$ satisfies
	\begin{equation}
		\left| \int f d\pi - \int f d\tilde{\pi}\right| \leq L W_p (\pi, \tilde{\pi}), \quad \forall \, \pi \in \Pi(\mu, \nu), \, \tilde{\pi} \in \Pi(\tilde{\mu}, \tilde{\nu}),
	\end{equation}
	with a constant $L > 0$. Then
	\begin{equation}\label{ineq:LDelta}
		\left| \inf_{\pi \in \Pi(\mu, \nu)} \int f d\pi - \inf_{\tilde{\pi} \in \Pi(\tilde{\mu}, \tilde{\nu})} \int f d\tilde{\pi}\right| \leq L [W^p_p (\mu, \tilde{\mu}) + W^p_p (\nu, \tilde{\nu})]^{1/p}.
	\end{equation}
\end{lemma}
\begin{proof}
	Denote a coupling $\pi^*$ attaining $\inf_{\tilde{\pi} \in \Pi(\tilde{\mu}, \tilde{\nu})} \int f d \tilde{\pi}$. The shadow $S(\pi^*)$ of $\pi^*$ is in $\Pi(\mu, \nu)$. Therefore,
	\begin{align*}
		& \inf_{\pi \in \Pi(\mu, \nu)} \int f d\pi - \inf_{\tilde{\pi} \in \Pi(\tilde{\mu}, \tilde{\nu})} \int f d \tilde{\pi} \leq \int f d S(\pi^*) - \int f d \pi^* \leq L W_p (S(\pi^*), \pi^*).
	\end{align*}
	Then we can apply Lemma \ref{lem:shadow} to obtain one side of \eqref{ineq:LDelta}. Another side follows similarly.
\end{proof}

\subsection{Properties of sub-root functions}
\begin{lemma}{\cite[Lemma G.5]{duan2021risk}}\label{lem:subroot}
	Suppose $\psi(\cdot)$ is a nontrivial sub-root function and $r^*$ is its positive fixed point, then
	\begin{itemize}
		\item[(1)] For any $C > 0$, $f(r) := C \psi(r/C)$ is sub-root and its positive fixed point $r_f$ satisfies $r_f = C r^*$;
		\item[(2)] For any $C > 0$, $f(r) : =C \psi(r)$ is sub-root and its positive fixed point $r_f$ satisfies $r_f \leq \max\{C^2, 1\} \times r^*$.
	\end{itemize}
\end{lemma}
\subsection{Properties of covering number}
\begin{lemma}\label{lem:diff_cover}
	Consider a class $\cF$ with a norm $\| \cdot \|$. Then $\cN(\varepsilon, \cF - \cF, \| \cdot \|) \leq \cN^2 (\varepsilon/2, \cF, \| \cdot \|)$, where $ \cF - \cF := \{f_1 - f_2: f_1, f_2 \in \cF \}$.
\end{lemma}

\begin{lemma}\label{lem:inf_cover}
	For a distribution $\beta$ and constant $p \in [1, \infty)$, we have $\cN (\varepsilon, \cF, \| \cdot \|_{L_p(\beta_N)}) \leq \cN_p(\varepsilon, \cF) \leq \cN_\infty(\varepsilon, \cF)$. Moreover, $\cN_\infty(\varepsilon, \cF) \leq \cN(\varepsilon, \cF, \| \cdot \|_\infty)$.
\end{lemma}

\subsection{Extension of H\"older continuous functions}
\begin{lemma}\label{lem:extension}
	Denote a subset $A \subseteq \R^n$. Consider a H\"older continuous function $f: A \rightarrow \R$ satisfying
	\begin{equation*}
		| f(x) - f(y) | \leq L | x - y|^\alpha, \quad \forall \, x, y \in A,  
	\end{equation*}
	with some $\alpha \in (0, 1]$ and $L > 0$.
	Define a function $g$ on $\R^n$ as
	\begin{equation}\label{eq:extend}
		g(x) := \inf_{y \in A} \{ f(y) + L |x - y|^\alpha \}, \quad x \in \R^n.
	\end{equation}
	Then $g: \R^n \rightarrow \R $ extends $f$ in the following sense:
	\begin{align*}
		g (x) & = f(x), \quad  \forall \, x \in A, \\
		| g(x) - g(y) | & \leq L | x- y|^\alpha, \quad  \forall \, x, y \in \R^n.
	\end{align*}
\end{lemma}
\begin{proof}
	With $\alpha \in (0, 1]$, we will use the fact that $|a + b|^\alpha \leq (|a| + |b|)^\alpha \leq |a|^\alpha + |b|^\alpha$.
	
	First, we show $g$ is finite. Since $f$ is finite and $g$ is defined with the infimum, we only need to show $g(x) > - \infty, \, x \in \R^n$. Note that when $y, y' \in A$, 
	\begin{align*}
		f(y) - f(y') + L |x - y|^\alpha \geq - L |y - y'|^\alpha + L |x - y|^\alpha \geq - L |x - y'|^\alpha.
	\end{align*}
	Hence,
	\begin{equation}\label{eq:lowerbd}
		g(x) = \inf_{y \in A} \{ f(y) + L |x - y|^\alpha \} \geq f(y') - L | x - y'|^\alpha > - \infty.
	\end{equation}
	
	Then, we prove $g = f$ on $A$. Suppose $x \in A$. With the definition of $g$, we have $g(x) \leq f(x)$ by taking $y = x$ in \eqref{eq:extend}. To show $g(x) \geq f(x)$, we choose $y' = x$ in \eqref{eq:lowerbd}.
	
	Next, we show the H\"older continuity of $g$ on $\R^n$. Given $\varepsilon > 0$, for any $x \in \R^n$, since $g$ is finite, the definition of $g$ guarantees that there exists $y \in A$ such that
	\begin{equation*}
		g(x) \geq f(y) + L|x - y|^\alpha - \varepsilon.
	\end{equation*} 
	Since $g(x') \leq f(y) + L | x' - y|^\alpha$ by the definition, we have
	\begin{align*}
		g(x) - g(x') & \geq f(y) + L |x - y|^\alpha - \varepsilon - f(y) - L |x' - y|^\alpha \\
		& \geq - L | x' - x|^\alpha - \varepsilon.
	\end{align*}
	Let $\varepsilon \rightarrow 0$, we get $g(x) - g(x') \geq - L | x' - x|^\alpha$. Interchanging the roles of $x$ and $x'$ shows the H\"older continuity.
	
\end{proof}
\end{document}